\documentclass{ecai} 

\usepackage{latexsym}
\usepackage{amssymb}
\usepackage{amsmath}
\usepackage{amsthm}
\usepackage{booktabs}
\usepackage{graphicx}
\usepackage{color}
\usepackage{amsmath}
\usepackage{amsthm}
\usepackage{booktabs}
\usepackage{algorithm}
\usepackage{algorithmic}
\usepackage[switch]{lineno}
\usepackage{paralist}
\usepackage{multirow}
\usepackage{ascmac, amssymb, bm, dsfont}

\newtheorem{theorem}{Theorem}

\newtheorem{proposition}[theorem]{Proposition}

\newcommand{\BibTeX}{B\kern-.05em{\sc i\kern-.025em b}\kern-.08em\TeX}
\newcommand{\R}{\mathbb{R}}

\newcommand{\argmin}{\mathop{\rm argmin}\limits}

\begin{document}


\begin{frontmatter}


\paperid{3495} 


\title{Learning Neural Strategy-Proof Matching Mechanism from Examples}


\author{\fnms{Ryota}~\snm{Maruo}\thanks{Corresponding Author. Email: mryota@ml.ist.i.kyoto-u.ac.jp}}
\author{\fnms{Koh}~\snm{Takeuchi}}
\author{\fnms{Hisashi}~\snm{Kashima}} 
\address{Kyoto University}


\begin{abstract}
  Designing two-sided matching mechanisms is challenging when practical demands for matching outcomes are difficult to formalize and the designed mechanism must satisfy theoretical conditions.
  To address this, prior work has proposed a framework that learns a matching mechanism from examples, using a parameterized family that satisfies properties such as stability.
  However, despite its usefulness, this framework does not guarantee strategy-proofness (SP), and cannot handle varying numbers of agents or incorporate publicly available contextual information about agents, both of which are crucial in real-world applications.
  In this paper, we propose a new parametrized family of matching mechanisms that always satisfy strategy-proofness, are applicable for an arbitrary number of agents, and deal with public contextual information of agents, based on the serial dictatorship (SD). 
  This family is represented by NeuralSD, a novel neural network architecture based on SD, where agent rankings in SD are treated as learnable parameters computed from agents' contexts using an attention-based sub-network.
  To enable learning, we introduce tensor serial dictatorship (TSD), a differentiable relaxation of SD using tensor operations.
  This allows NeuralSD to be trained end-to-end from example matchings while satisfying SP.
  We conducted experiments to learn a matching mechanism from matching examples while satisfying SP. 
  We demonstrated that our method outperformed baselines in predicting matchings and on several metrics for goodness of matching outcomes.
\end{abstract}

\end{frontmatter}


\section{Introduction}
\label{sec:introduction}
Two-sided matching is a problem that constructs a set of pairs from two groups of agents based on the preferences that one group of agents has for the other.
Various real-world applications of two-sided matching exist such as labor markets for medical interns and residents~\cite{roth1984evolution}, assignments in schools~\cite{abdulkadiroglu2003school}, and loan markets between firms and banks~\cite{chen2013loan}.
Recently, publicly available contextual information about agents has been increasingly used to improve matchings in socially significant settings.  
For example, school choice in Brazil~\cite{aygun2021college,dur2020explicit} and matching children with daycare centers~\cite{sun2023daycare} have attempted to achieve more socially desirable outcomes in access to education and childcare using agents' contextual information.

A matching mechanism is required to balance the notions of goodness in matching outcomes and theoretical guarantees of the conditions it must satisfy.
Mechanism designers have developed various mechanisms for two-sided matching~\cite{roth1990two-sided}, which aim to optimize simple objectives aligned with desirable matching outcomes, such as maximizing social welfare.
A mechanism is also required to satisfy conditions like strategy-proofness (SP)~\cite{roth1982economics,roth1990two-sided}, which ensures that no agent can benefit from misrepresenting their preferences.
However, existing mechanisms do not always satisfy practical demands that are difficult to formalize simply~\cite{narasimhan2016automated}. 
For example, ethical requirements are often too complex to be represented by mathematical formulas~\cite{li2017ethics,sonmez2022market}.
As a result, practitioners construct ad-hoc mechanisms or empirical matching rules that reflect their implicit goodness, but do not satisfy theoretical conditions.
Satisfying both practical demands and theoretical guarantees is an urgent challenge. 

To address the aforementioned challenge, a general framework has been introduced by \citet{narasimhan2016automated} that learns a matching mechanism from example matchings generated by ad-hoc mechanisms or empirical rules, assuming that these examples reflect the practitioner's implicit objectives.
They proposed a parameterized family of matching mechanisms that satisfy the condition called stability~\cite{roth1990two-sided}, and used supervised learning to estimate parameters from examples to develop a mechanism that satisfies both practical demands and stability.
However, they do not address several critical issues that arise in practical and socially sensitive applications.
First, they do not consider strategy-proofness (SP), a condition that is crucial in socially sensitive applications like affirmative action~\cite{aygun2021college}, because all stable mechanisms are inherently not SP~\cite{roth1982economics,roth1990two-sided}.
Second, their framework assumes the number of agents to be constant, and cannot be applied to varying numbers of agents~\cite{nrmp2024impact}.
Third, they do not incorporate publicly observable contextual information about agents, which is essential for socially beneficial matchings~\cite{aygun2021college,dur2020explicit}.

In this paper, we propose a new framework utilizing examples for estimating matching mechanisms that always satisfy SP.
To address the aforementioned problems, we propose a neural serial dictatorship (NeuraSD), a novel neural network that represents a parameterized family of SP matching mechanisms.
NeuralSD is based on serial dictatorship (SD)~\cite{abdulkadiroglu1998random,satterthwaite1981strategy}, an SP mechanism in which agents are sequentially ordered according to a ranking, and each agent selects their most preferred unmatched partner or remains single.
We treat the agent ranking in SD as a learnable parameter, and develop NeuralSD to output a matching outcome by applying the SD procedure to a given preference profile and agents' contextual information.
To generate the ranking in a context-aware manner while satisfying SP, we introduce a sub-network that computes the ranking solely based on agents’ contexts. 
This ensures SP under the assumption that contexts are not manipulated, which is reasonable in socially sensitive settings.
This sub-network uses an attention-based architecture~\cite{vaswani2017attention} to make it invariant to the number of agents.
Because SD involves discrete operations that are non-differentiable, we develop a differentiable relaxation called tensor serial dictatorship (TSD), which simulates the SD process using a sequence of tensor operations.
During training, NeuralSD uses the sub-network to estimate rankings from contexts and applies TSD to output matchings that align with example-based supervision.
At inference time NeuralSD applies the original SD algorithm with rankings estimated from agent contexts. 

We conducted experiments to learn matching mechanisms from examples while satisfying SP.
We generated examples by existing mechanisms using synthetic agent preferences and agent contexts to obtain matching outcomes. 
We used random serial dictatorship (RSD) as a baseline, which is a variant of SD with a randomly selected ranking and well-known SP mechanism~\cite{abdulkadiroglu1998random}.
To evaluate the learning performance via examples, we employed discrepancies between the predicted matchings and original matchings.
Furthermore, we used other relevant metrics to verify that NeuralSD can reflect the implicit goodness encoded in the example matchings.
In small-scale experiments, we also evaluated the optimality of the ranking estimation in NeuralSD and confirmed that NeuralSD generates better rankings than RSD.
For larger numbers of agents, NeuralSD outperformed baselines in terms of matching discrepancies and other metrics for evaluating goodness of estimated matching outcomes. 

The contributions of this study are as follows.
(i) We propose TSD, which accurately executes SD using tensor operations and enables us to back-propagate the loss between the predicted and target matchings.
(ii) We developed a neural network called NeuralSD to represent a parameterized family of SP matching mechanisms, by combining TSD with an attention-based sub-network to reflect the agent contexts in the resulting matching.
NeuralSD is context-aware, rigorously SP, and can learn from datasets with any number of agents to perform matchings at any scale.
(iii) We conducted experiments to learn matching mechanisms from examples. 
The results show that estimating agent rankings from public contextual information leads to better matching outcomes than using random rankings.
This demonstrates that NeuralSD can recover mechanisms through the example matchings that reflect the implicit notions of goodness for outcomes.


\section{Related Work}
\label{sec:related_work}

Prior works have designed various matching mechanisms to satisfy certain theoretical conditions, other than SP and SD mechanism.
For example, individual rationality (IR) is a standard condition, which means that no agent is matched with someone who is worse than remaining unmatched~\cite{roth1990two-sided}.
Stability, which ensures IR and that no pair of agents would prefer to rematch with each other, is also a well-studied condition~\cite{roth1990two-sided}.
The deferred acceptance (DA) mechanism has been designed to satisfy stability~\cite{gale1962college}.

Our research contributes to the field of {\em automated mechanism design} (AMD), which has been first introduced by \citet{conitzer2002complexity,conitzer2004self} to automatically solve mechanism design problems.
Specifically, AMD focuses on optimizing the expectations of functions that explicitly formulate social objectives~\cite{conitzer2002complexity,conitzer2004self,sandholm2003automated}, and most subsequent research has followed this framework.
For instance, in auction settings, one of the primary social objectives is revenue~\cite{conitzer2004self,curry2023differentiable,duan2023scalable,duan2022context,duetting2019optimal,feng2018deep,peri2021preferencenet,rahme2021permutation,sandholm2015automated,shen2019automated}.
Other social objectives are also considered, such as negative sum of agents' costs \cite{golowich2018deep} in facility location, social welfare \cite{wang2020mechanism} in public project settings, and worst-case efficiency loss \cite{guo2015social} in general design settings.

In the field of differentiable economics \cite{duetting2019optimal}, researchers have designed neural networks to solve AMD problems.
Pioneered by \citet{duetting2019optimal}, who proposed RegretNet, a line of AMD research has focused on solving revenue-optimal auction problems using neural networks~\cite{curry2023differentiable,feng2018deep,peri2021preferencenet,rahme2021permutation,shen2019automated}; however, the current study focuses on two-sided matching problems.
Neural networks that consider contextual information in auction settings have been proposed~\cite{duan2022context,duan2023scalable}, but none of them consider models for two-sided SP matching mechanisms.
\citet{ravindranath2023deep} have proposed a neural network for two-sided matching problems to explore the trade-offs between SP and stability, but their model was not rigorously SP, did not incorporate contextual information, and only worked with a fixed number of agents.
\citet{curry2022learning} have investigated differentiable matchings in neural network-based auction design, but in contrast to our approach, their matchings did not incorporate contextual information.

While the approaches and neural network architectures discussed above aim to search for new mechanisms by optimizing an explicit objective, our focus is on designing mechanisms that reflect implicit goodness for matching outcomes while still satisfying strategy-proofness (SP), rather than optimizing a simple social objective.
To this end, we adopt the supervised learning framework proposed by~\citet{narasimhan2016automated} and~\citet{narasimhan2016general}, which allows us to avoid the need to explicitly specify such objectives.

\section{Background}
\label{sec:background}
We use $[n]$ to denote a set $\{1,\dots, n\}$ for any $n \in \mathbb{N}$. 
We represent a row vector as $\bm{x}= [x_1, \dots, x_d]$, a matrix as $\bm{X} = [X_{i,j}]_{ij}$, and a three-dimensional tensor as $\mathbf{X} = [X_{i,j,k}]_{i,j,k}$.

{\bf Matching Instance.}
We focus on one-to-one, two-sided matchings.
We refer to each side as `workers' and `firms' \cite{ravindranath2023deep}.
Let $W := \{w_1,\dots, w_n\}$ and $F := \{f_1,\dots, f_m\}$ denote the sets of workers and firms, respectively.
$W \cup F$ forms the set of {\em agents}. 
As in several applications~\cite{aygun2021college,dur2020explicit,sun2023daycare}, each agent $a\in W\cup F$ is equipped with $d$-dimensional publicly available {\em contexts}, which represent attributes of agents, such as gender, socioeconomic status, and other relevant characteristics that are publicly known.
We denote the contexts of workers and firms by $\bm{X}_W\in \R^{n\times d}$ and $\bm{X}_F\in\R^{m\times d}$, respectively.

Each agent has a preference over agents in the other side as well as the {\em unmatch} option.
We use $\perp$ to indicate the {\em unmatch} option and define $\overline{W} := W \cup \{\perp\}$ and $\overline{F} := F \cup \{\perp\}$.
A preference profile is denoted by ${\succcurlyeq} := (\succcurlyeq_{w_1},\dots, \succcurlyeq_{w_n},\succcurlyeq_{f_1},\dots, \succcurlyeq_{f_m})$, which is a tuple of the preferences of all the agents.
For a worker $w_i$, the preference $\succcurlyeq_{w_i}$ is a linear order, that is, a complete, transitive, and anti-symmetric binary relation over $\overline{F}$.\footnote{
    A {\em binary relation} $R$ over a set $X$ is a subset of $X \times X$. 
    We define $xRy\iff (x,y)\in R$, as per convention.
    The binary relation $R$ is {\em complete} if for any pair $(x,y) \in X \times X$, either $xRy$ or $yRx$. 
    $R$ is {\em transitive} if, for all $x, y, z \in X$, the conditions $xRy$ and $yRz$ imply that $xRz$. 
    $R$ is {\em anti-symmetric} if, for all $x, y \in X$ where $x \neq y$, $xRy$ implies $\lnot (yRx)$. 
    Note that if $R$ is complete, then $xRx$ for all $x \in X$.
}
Similarly, for a firm $f_j\in F$, $\succcurlyeq_{f_j}$ is a preference over $\overline{W}$.
Note that for all agents $a, a'$, we have $a' \succcurlyeq_a a'$ because $\succcurlyeq_a$ is a complete binary relation.
$P_{W}$ and $P_{F}$ denote the sets of all possible preferences over $\overline{F}$ and $\overline{W}$, respectively.
We define $P := (\prod_{i=1}^n P_{W})\times (\prod_{j=1}^m P_{F})$ as the set of all possible preference profiles.
For a preference profile ${\succcurlyeq}\in P$, we say that an agent $a$ {\em strictly prefers} $b$ over $b'$ under the preference $\succcurlyeq_{a}$ if $b\neq b'$ and $b \succcurlyeq_{a} b'$.
We denote $b \succ_{a} b'$ if $a$ strictly prefers $b$ over $b'$
We denote by $\succcurlyeq_{-a}$ the preference profile excluding the agent $a\in W\cup F$ from a profile $\succcurlyeq$, and $(\succcurlyeq_a', \succcurlyeq_{-a}) := (\succcurlyeq_{w_1}, \dots, \succcurlyeq_a', \dots, \succcurlyeq_{f_m})$ denotes the preference profile obtained by replacing $\succcurlyeq_a$ with $\succcurlyeq_a'$.

We represent a matching problem by an {\em instance} $I$.
An instance is denoted by a tuple $I := \langle W_I, F_I, \bm{X}_{W_I}, \bm{X}_{F_I} \rangle$, which consists of the set of workers $W_I$, the set of firms $F_I$, the public contextual information $\bm{X}_{W_I}$ of the workers, and the public contextual information $\bm{X}_{F_I}$ of the firms.
For simplicity, we omit the subscript $I$ from components (e.g., $W$, $F$, $\bm{X}_W$, $\bm{X}_F$) unless specified.
We denote the set of all possible instances by $\mathcal{I}$.
We define the set of all preference profiles across all instances as $\mathcal{P} := \cup_{I\in\mathcal{I}} P_I$, where $P_I$ is the set of all possible preference profiles for a given instance $I$.

{\bf Matching.}
Given an instance with a set of agents $W\cup F$, we represent a matching by a function $\mu: W\cup F\to\overline{W}\cup\overline{F}$.
A function $\mu$ is a {\em matching} if, for all workers $w_i, w_{i'} \in W$ with $i \neq i'$ and all firms $f_j, f_{j'} \in F$ with $j \neq j'$, the following conditions hold:
\begin{enumerate}[(1)]
    \item Every agent is either matched to a partner or remains unmatched: $\mu(w_i)\in \overline{F}$ and $\mu(f_j)\in \overline{W}$;
    \item One-to-one matching for workers: if $\mu(w_i)\neq \perp$, then $\mu(w_i) \neq \mu(w_{i'})$ and $\mu(\mu(w_i)) = w_i$;
    \item One-to-one matching for firms: if $\mu(f_j)\neq \perp$, then $\mu(f_j) \neq \mu(f_{j'})$ and $\mu(\mu(f_j)) = f_j$.
\end{enumerate}
We say that an agent $\mu(a)$ is the {\em partner} of an agent $a$.
Following \citet{ravindranath2023deep}, we also represent a matching $\mu$ by a binary matrix $\bm{M}$.\footnote{
  This is identical to the binary version of the marginal probability matrix by \citet{ravindranath2023deep}.
}
In other words, $M_{ij} = 1$ if and only if $\mu(w_i) = f_j$, and $M_{ij}=0$ otherwise. 
Therefore, each row and column of $\bm{M}$ has a sum to $1$, except for the last row and column; that is, $\sum_{i=1}^{n+1}M_{ij} = 1$ for all $j\in [m]$ and $\sum_{j=1}^{m+1} M_{ij} = 1$ for all $i \in [n]$.
We fix $M_{n+1,m+1} = 0$.
If a matrix $\bm{M}$ represents a matching, then we call it a {\em matching matrix} or simply a matching.
$\mathcal{M}_I$ denotes the set of all possible matchings given the instance $I$.

We define several fundamental properties related to matchings.
Let $\succcurlyeq$ be a preference profile.
A matching $\mu$ is {\em individually rational} (IR) under the preference profile $\succcurlyeq$ if no agent is matched with someone who is worse than the unmatch option: for all agents $a\in W\cup F$, $\mu(a) \succcurlyeq_a \perp$.
A matching $\mu$ is {\em stable} if it is IR and no pair of agents would prefer to rematch with each other.
A pair $(w_i, f_j) \in W \times F$ is a {\em blocking pair} under a matching $\mu$ and a preference profile $\succcurlyeq$ if both agents prefer each other to their matched partners under $\mu$, that is, 
$f_j \succ_{w_i} \mu(w_i)$ and $w_i \succ_{f_j} \mu(f_j)$.
Then, a matching is stable with respect to the preference profile $\succcurlyeq$ if it is IR and has no blocking pairs under the profile.
A matching $\mu'$ {\em Pareto dominates} $\mu$ according to the preference profile $\succcurlyeq$ if $\mu'$ makes all the agents better off that $\mu$ with respect to $\succcurlyeq$: for every agent $a \in W \cup F$, $\mu'(a) \succcurlyeq_a \mu(a)$, and for at least one agent $a' \in W \cup F$, $a'$ strictly prefers $\mu'(a')$ to $\mu(a')$, i.e., $\mu'(a') \succ_{a'} \mu(a')$.
A matching $\mu$ is {\em Pareto efficient} with respect to $\succcurlyeq$ if no other matching $\mu'$ {\em Pareto dominates} $\mu$ according to the profile $\succcurlyeq$.
We consider Pareto efficiency with respect to both sides of the agents, as is commonly done in the literature~\cite{erdil2017two-sided,roth1990two-sided}.

{\bf Matching Mechanism.}
A {\em matching mechanism} is a function $g$ that maps an instance $I$ and a preference profile ${\succcurlyeq}\in P_I$ to a matching $g(I,\succcurlyeq)\in\mathcal{M}_I$.
For notational convenience, we denote the partner of an agent $a$ by $g(I, \succcurlyeq)_a$ instead of $g(I, \succcurlyeq)(a)$.

Given a matching mechanism, agents strategically make preference reports and then a matching outcome is determined.
Agents are assumed to be rational and self-interested, concerned solely with whom they match. 
Formally, given an instance $I$, the payoff function $u_a:\mathcal{M}_I\to\R$ of an agent $a$ is defined as $u_a(\mu) := \#\{b\mid \mu(a)\succcurlyeq_{a} b\}$.
Because we consider the agents' contexts $\bm{X}_W$ and $\bm{X}_F$ are public, we assume that agents cannot misreport their contexts.
On the other hand, agents can misreport their preferences to be matched with more desirable partners.
Formally, the set of possible preference reports for $w_i$ is $P_W$, i.e., the set of all preferences over $\overline{F}$.
Similarly, the set of possible preference reports for each firm $f_j$ is $P_F$, i.e., the set of all preference over $\overline{W}$.
Based on the matching mechanism $g$, we assume that agents behave according to the following scenario:
\begin{enumerate}[(1)]
    \item An instance $I$ realizes from $\mathcal{I}$, and both the mechanism designer and agents observe the set of agents $W\cup F$ as well as the agents' public contextual information $\bm{X}_{W}$ and $\bm{X}_{F}$. 
    \item Simultaneously, each agent $a$ observes their own preference $\succcurlyeq_{a}$ over agents in the opposite side.
    Preferences are private information of agents and not observable by the mechanism designer.
    \item Agents report their own preferences to form a profile $\succcurlyeq'\in P_I$.
    \item The mechanism designer collects $\succcurlyeq'$ to determine a matching outcome as $g(I, \succcurlyeq')\in \mathcal{M}_I$.
\end{enumerate}
Specifically, $g$ may depend on the public contextual information $\bm{X}_W$ and $\bm{X}_F$ to yield better matching outcomes.

We now define several conditions for a matching mechanism.
Because agents can misreport their preferences to obtain more favorable partners, a matching mechanism should be designed to prevent such manipulation.
A matching mechanism $g$ is {\em strategy-proof} (SP) if it is a dominant strategy for each agent to truthfully report his or her preferences.
Formally, a matching mechanism $g$ is SP if for all instances $I = \langle W, F, \bm{X}_W, \bm{X}_F\rangle\in\mathcal{I}$, for all agent $a\in W\cup F$, and for all untruthful preference reports $\succcurlyeq_a'\in P_{W}\cup P_{F}$, no agent can benefit from misreporting their preferences, defined as follows:
\begin{align}
    g(I, (\succcurlyeq_a, \succcurlyeq_{-a}))_a \succcurlyeq_{a} g(I, (\succcurlyeq_a', \succcurlyeq_{-a}))_a.\label{eq:def_strategy_proof}
\end{align}
Importantly, we consider SP only for the reports and not for the public contexts, because we assume that the contextual information are public and known to the mechanism designer.
A matching mechanism is {\em IR} or {\em stable} if it always outputs IR or stable matchings with respect to the reports, respectively.
A mechanism is {\em Pareto efficient} if every output is a Pareto efficient matching with respect to the reports.
We consider IR, stability and Pareto efficiency with respect to the reported preference profile, as commonly studied in the literature~\cite{roth1990two-sided}. 

\section{Problem Setting}
\label{sec:problem_setting}
We consider the problem of estimating a desirable matching mechanism from examples provided by a mechanism designer, under the constraint that the learned mechanism must satisfy SP.

To formalize this setting, we assume access to a dataset of example matchings generated by ad-hoc mechanisms or empirical rules, which are assumed to reflect the practitioner’s implicit objectives.
Formally, we assume access to a dataset of example matchings $S := \{((I^\ell, {\succcurlyeq^\ell}), \bm{M}^\ell)\}_{\ell=1}^L$, where each example consists of an instance $I^\ell$, a preference profile $\succcurlyeq^\ell$, and the corresponding matching matrix $\bm{M}^\ell$.
We assume that each matching satisfies $\bm{M}^\ell = g(I^\ell, \succcurlyeq^\ell)$, where $g$ is an unknown mechanism corresponding to an ad-hoc or empirical rule.
That is, these examples are assumed to encode the designer’s practical demand for matching outcomes which are diffucult to formalize explicitly, such as ethical requirements~\cite{li2017ethics,sonmez2022market}.
We further assume that instances and preference profiles are independently and identically distributed samples from an unknown distribution over $\mathcal{I}\times\mathcal{P}$.
Because the mechanism designer's notion of goodness for matching outcomes is not explicitly specified, the corresponding mechanism $g$ cannot be expressed in closed form.

Such a dataset $S$ can be constructed from a history of past matchings~\cite{narasimhan2016automated}.  
In this setting, a mechanism designer collects records of past matchings, where each record consists of an instance $I^\ell$ and a reported preference profile $\succcurlyeq^\ell$, forming the input pair $(I^\ell, \succcurlyeq^\ell)$.  
The associated matching outcome $\bm{M}^\ell$ can either be the actual outcome produced in the past, or a new outcome that the designer now considers more desirable for the same input.  
This approach enables the designer to convey their implicit matching objectives through examples, even without a formal specification.

Our goal is to design a matching mechanism that reflects the implicit goodness of matching outcomes expressed in the example matchings while ensuring SP. 
To this end, given the example matchings $S$, we aim to estimate a matching mechanism $f_{\bm{\theta}^*}$ that approximates the underlying mechanism $g$ through the examples, by searching over a parameterized family of SP mechanisms $\mathcal{F} := \{f_{\bm{\theta}} \mid \bm{\theta} \in \bm{\Theta}\}$.  
We formulate this as the following optimization problem:
\begin{align}
\bm{\theta}^* = \argmin_{\bm{\theta} \in \bm{\Theta}} \sum^L_{\ell=1} \mathcal{L}(f_{\bm{\theta}}(I^\ell, \succcurlyeq^\ell), \bm{M}^\ell),
\label{eq:search_over_theta}
\end{align}
where $\mathcal{L}$ is a discrepancy function between two matching matrices, and $\bm{\Theta}$ is the parameter space.

Our goal is to obtain an SP matching mechanism that best approximates the example matchings.  
To this end, We use a family $\mathcal{F}$ of SP matching mechanisms to approximate the example outcomes, which reflect an implicit notion of goodness for matching outcomes.
Because we cannot obtain a complete characterization of all SP matching mechanisms, we construct $\mathcal{F}$ as a subset of all SP mechanisms, as motivated by \citet{narasimhan2016automated}.
Note that we focus on $\mathcal{F}$ that includes mechanisms that can learn from $S$ and excludes non-adaptive mechanisms, such as one that first estimates preferences from contexts and then applies a fixed algorithm like deferred acceptance (DA)~\cite{gale1962college}, to the estimated preferences while ignoring reports.
In addition, $S$ can be constructed under the assumption that the reported preferences are truthful because we consider SP mechanisms.
During inference, SP ensures that agents report preferences truthfully, and hence, the model $f_{\bm{\theta}^*}$ trained on $S$ will yield an approximately desirable matching outcome.

\section{Proposed Method}
\begin{figure*}
    \centering
    \includegraphics[width=1.0\linewidth]{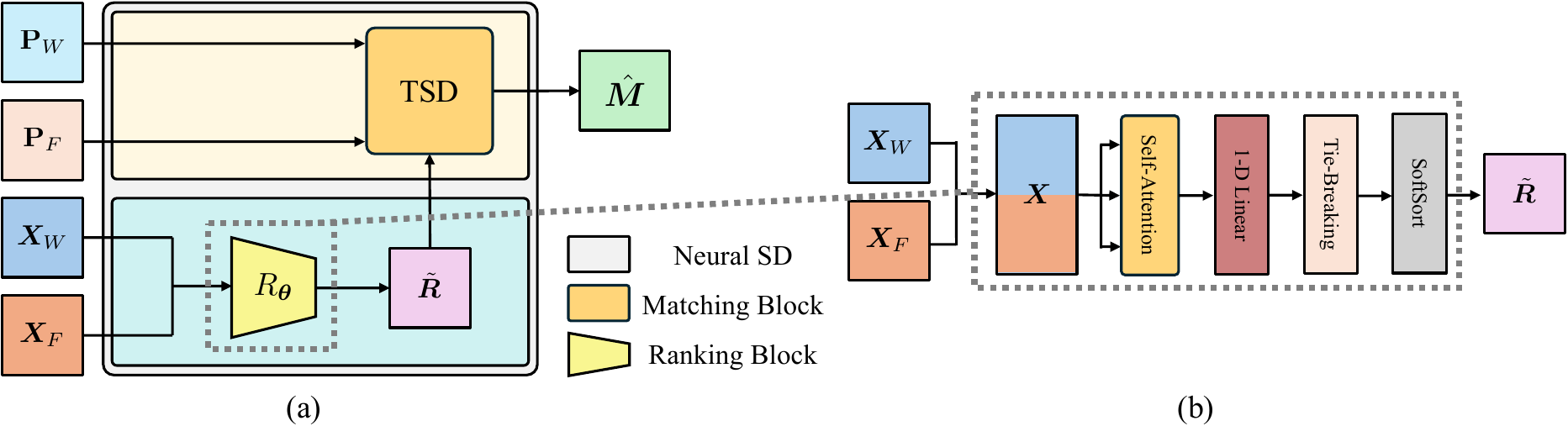}
    \caption{
        NeuralSD architecture. 
        (a) Overview of the entire architecture. 
        The orange block labeled `TSD' represents the matching block, which internally executes TSD.
        The block labeled `$\bm{R}_{\bm{\theta}}$' is the ranking block.
        (b) Detailed structure of the ranking block.
    }
    \label{fig:neuralsd}
\end{figure*}
\label{sec:proposed_method}

To solve the problem in Equation~\eqref{eq:search_over_theta}, we construct a parametrized family of SP matching mechanisms $\mathcal{F}$ based on serial dictatorship (SD)~\cite{abdulkadiroglu1998random,satterthwaite1981strategy}, one of the SP matching mechanisms. 
In the procedure of SD, agents are ordered sequentially with respect to a ranking and each agent either selects their most preferred unmatched partner or remains single. 
Because SD depends on the agent ranking, we propose to model $f_{\bm{\theta}}$ in Equation~\eqref{eq:search_over_theta} through a neural network where $\bm{\theta}$ corresponds to agent rankings.
Concretely, we make SD learnable by a neural network that first parametrically computes an agent ranking from agents' public contextual information, and then outputs a matching outcome based on SD and the ranking.
This architecture satisfies SP because the ranking solely depends on the public contexts.
On the other hand, SD cannot be directly used in a neural network due to discrete operations. 
Therefore, we first make the procedure of SD differentiable with respect to the agent ranking so that we can use a standard gradient-based learning method. 
Then, we develop a sub-network that parametrically computes the agent ranking.

We first describe SD and its properties and introduce TSD for executing SD using tensor operations. 
We then develop NeuralSD along with the parametric computation of agent rankings.

\subsection{Serial Dictatorship (SD)}
\label{sec:serial_dictatorship}
We describe SD in two-sided matching.
SD is widely applied for one-sided matching, but also can be considered in two-sided matching~\cite{ravindranath2023deep}.
We represent an agent order as a {\em ranking} $\bm{r} = (r_1,\dots, r_{n+m})$, where $r_k = a$ indicates that agent $a$ is assigned to rank $k$. 
Given a ranking $\bm{r}$ and reports $\succcurlyeq$, SD proceeds as follows:
\begin{enumerate}[(1)]
\item Initialize the matching $\mu$ to be empty.
\item For each round $k = 1, \dots, n+m$, if $r_k$ is not yet matched in $\mu$, then assign $\mu(r_k)$ to either the most preferred unmatched agent under $\succcurlyeq$ or to the unmatch option $\perp$, particularly if the remaining agents are not strictly preferred to $\perp$.
The agent $r_k$ cannot designate agents that are already matched to $\perp$.
\end{enumerate}

SD satisfies SP and Pareto efficiency.
\citet{ravindranath2023deep} did not formally prove that SD is both SP and Pareto efficient; therefore, we offer the following statement. 
This is similar to the result by \citet{imamura2024matching}, but they considered one-sided matching, while we consider two-sided settings.
\begin{proposition}
    SD based on an agent ranking $\bm{r}$ is both SP and Pareto efficient for all rankings $\bm{r}$.\label{prop:sd}
\end{proposition}
The proof is provided in Appendix~\ref{apdx:prop_sd}.
Note that $\bm{r}$ should be determined before collecting reports to satisfy SP.

SD is not IR because an agent could be matched with one less desirable than $\perp$.
Therefore, we quantify the extent to which SD fails to be IR by using IR violation (IRV) proposed by~\citet{ravindranath2023deep}. 
For a matching matrix $\bm{M}$ and reports $\succcurlyeq$, IRV is denoted as $irv(\bm{M}, \succcurlyeq)$ and ranges from $0$ (IR under $\succcurlyeq$) to $1$.
The definition is provided in Appendix~\ref{apdx:experimental_details}.
We show the upper bound for IRV of SD and provide the proof in Appendix~\ref{apdx:prop_irv_bound_sd}.
\begin{proposition}
    Let $I$ be an instance and $\bm{M}_{\bm{r}}$ be a matching matrix of SD on an arbitrary agent ranking $\bm{r}$.
    Then, $\max_{I\in\mathcal{I}}\{irv(\bm{M}_{\bm{r}},\succcurlyeq_I)\} \le 1/2$ for all rankings $\bm{r}$.
    \label{prop:irv_upper_bound_sd}
\end{proposition}

\subsection{SD as Tensor Operations}
\label{subsec:TSD}
Because SD is a deterministic procedure with several discrete operations, it is non-differentiable and unsuitable as a neural network layer.
To address this issue, we propose tensor serial dictatorship (TSD), a differentiable variant of SD that exactly simulates the original algorithm using tensor operations and enables backpropagation of the errors from example matchings, as in Equation~\eqref{eq:search_over_theta}. 
We outline the key components below, and full pseudo-code is provided in Appendix~\ref{apdx:tensorserialdictatorship}.

TSD takes three tensors as input: two represent the agents' reported preferences, and one represents an agent ranking. 
Let the instance consist of $n$ workers and $m$ firms, with preference profiles $\succcurlyeq$, and a ranking $\bm{r} = (r_1, \dots, r_{n+m})$ over the agents.
The first two tensors, $\mathbf{P}_W \in \{0,1\}^{n \times (m+1) \times (m+1)}$ and $\mathbf{P}_F \in \{0,1\}^{m \times (n+1) \times (n+1)}$, represent the reported preferences of workers and firms, respectively. 
Specifically, $\mathbf{P}_W = [\bm{P}_{w_1},\dots,\bm{P}_{w_n}]$ comprises $n$ matrices, where $\bm{P}_{w_i}=[P_{i,j,k}]_{j,k\in[m+1]}$ is a {\em preference matrix} of the worker $w_i$ that represents the preference $\succcurlyeq_{w_i}$ as:
\begin{align}
    P_{i,j,k} := \mathbb{I}[&\text{$j \leq m$ and $\mathrm{ord}(f_j,\succcurlyeq_{w_i})=k$, or} \notag\\
    &\text{$j = m+1$ and $\mathrm{ord}(\perp,\succcurlyeq_{w_i})=k$}].\label{eq:defP}
\end{align}
Here, $\mathbb{I}$ is the indicator function, and $\mathrm{ord}(\overline{f},\succcurlyeq_{w_i}) := \#\{\overline{f'}\in\overline{F} \mid \overline{f'} \succcurlyeq_{w_i} \overline{f}\}$ is the order of $\overline{f}\in\overline{F}$ in $\succcurlyeq_{w_i}$.
Similarly, we define $\mathbf{P}_F$ comprising the preference matrices of $m$ firms.
We represent preferences as above, but one can refine the representations to reflect preference domain structures.
The third input, $\bm{R}=[R_{i,j}]_{i,j\in[n+m]} \in \{0,1\}^{(n+m) \times (n+m)}$, is a permutation matrix called {\em ranking matrix} that represents the ranking $\bm{r}$ as follows:
\begin{align}
    R_{i,j} := \mathbb{I}[&\text{$1\le i\le n$ and $r_j = w_i$, or}\notag\\
    &\text{$n+1\le i\le n+m$ and $r_j = f_{i-n}$}].\label{eq:defR}
\end{align}

{\bf TSD algorithm.}
The TSD algorithm is a matrix representation of SD that operates with the three tensors $\mathbf{P}_W$, $\mathbf{P}_F$, and $\bm{R}$.
We denote it as a function ${\rm TSD}(\mathbf{P}_W, \mathbf{P}_F, \bm{R})$.
It updates the match step by step based on $\bm{R}$ and removes used options from $\mathbf{P}_W$ and $\mathbf{P}_F$ to prevent repeated matches.
The ${\rm TSD}(\mathbf{P}_W, \mathbf{P}_F, \bm{R})$ algorithm is as follows:
\begin{enumerate}[(1)]
    \item Initialize $\bm{M}$ to be a zero matrix and set $k := 1$.
    \item Extract the $k$-th leftmost column from $\bm{R}$, and use it to compute the preference matrix for the agent $r_k$.
    \item Compute the new match using the $r_k$'s preference matrix $\bm{P}_{r_k}$ and update $\bm{M}$ based on the result.
    \item If $r_k$ is a worker $w_i$, then set all entries in the $i$-th row of $\bm{P}_{f_1}, \dots, \bm{P}_{f_m}$ to zero. 
    If $r_k$ is a firm, apply the same procedure to $\bm{P}_{w_1}, \dots, \bm{P}_{w_n}$.
    \item If $r_k$ is a worker $w_i$ matched with $f_j$, then set the $j$-th row of $\bm{P}_{w_1}, \dots, \bm{P}_{w_n}$ to zero and set $\bm{P}_{f_j}$ to a zero matrix. 
    If $w_i$ is matched with $\perp$, then do nothing. 
    If $r_k$ is a firm, apply the same procedure to $\bm{P}_{f_1}, \dots, \bm{P}_{f_m}$.
    \item If $k = n+m$, output $\bm{M}$. 
    Otherwise, increment $k$ to $k+1$ and return to step (2).
\end{enumerate}
A conceptual visualization of this procedure is shown in Figure~\ref{fig:tsd} of Appendix~\ref{apdx:tensorserialdictatorship}.
We prove that this procedure exactly computes the same matching result as that of SD.
\begin{proposition}
    Consider an instance with $n$ workers and $m$ firms, matched by SD according to the ranking $\bm{r} = $ $(r_1,$ $\dots,$ $r_{n+m})$. 
    If the three tensors $\mathbf{P}_W$, $\mathbf{P}_F$, and $\bm{R}$ are defined as per Equations \eqref{eq:defP} and \eqref{eq:defR}, then TSD outputs the matrix representation of the matching outcome determined by this SD.
    \label{prop:tsd}
\end{proposition}
Assuming $n\ge m$, without loss of generality, this computation requires $O(n^4)$.
The proof of Proposition~\ref{prop:tsd} and the computational complexity are provided in Appendix~\ref{apdx:tensorserialdictatorship}.
Although the outputs of SD and TSD are identical, we use TSD in our neural network because we cannot directly use SD as a layer because it is non-differentiable and thus unsuitable for gradient-based learning method.

\subsection{NeuralSD}
\label{subsec:neuralsd}
{
\begin{table*}[t]
    \centering
    \caption{
        Mean and std for tests at $n=10,120,200$ with different mechanisms for generating example matchings.
        HD, \#BP, SV, RW, and IRV are the evaluation metrics.
        Arrows show if higher ($\uparrow$) or lower ($\downarrow$) values are better.
        `$\times 10^{-k}$' means values are scaled by $10^{-k}$.
        Bold font indicates the best means.
    }
    \label{tab:results}
    \scalebox{0.8}{
    \begin{tabular}{clrrrrrrrr}
    \toprule
    &&\multicolumn{4}{c}{Deferred Acceptance (DA)}&\multicolumn{2}{c}{Equal-weighted Hungarian (EH)}&\multicolumn{2}{c}{Minority-weighted Hungarian (MH)}\\\cmidrule(lr){3-6}\cmidrule(lr){7-8}\cmidrule(lr){9-10}
    $n$&Model&HD$\downarrow(\times 10^{-1})$&\#BP$\downarrow(\times 10^{-2})$&SV$\downarrow(\times 10^{-2})$&IRV$\downarrow$&HD$\downarrow(\times 10^{-1})$&RW$\uparrow$&HD$\downarrow(\times 10^{-1})$&RW$\uparrow$\\\cmidrule(lr){1-1}\cmidrule(lr){2-2}\cmidrule(lr){3-6}\cmidrule(lr){7-8}\cmidrule(lr){9-10}

    $10$&RSD (baseline)&$4.61\pm 1.23$&$12.3\pm 5.18$&$1.30 \pm 0.815$&$0.00\pm0.00$&$4.56\pm1.17$&$0.918\pm0.0359$&$4.66\pm1.20$&$0.910\pm0.0410$\\
    &NeuralSD (ours)&$\mathbf{4.57\pm 1.25}$&$\mathbf{11.1\pm 5.48}^\ddagger$&$\mathbf{1.04\pm 0.813}^\ddagger$&$0.00\pm0.00$&$\mathbf{4.35\pm1.23}^\ddagger$&$\mathbf{0.930\pm0.0372}^\ddagger$&$\mathbf{4.33\pm1.29}^\ddagger$&$\mathbf{0.925\pm0.0418}^\ddagger$\\\cmidrule(lr){1-1}\cmidrule(lr){2-2}\cmidrule(lr){3-6}\cmidrule(lr){7-8}\cmidrule(lr){9-10}
    $120$&RSD&$5.73\pm0.258$&$10.0\pm1.27$&$0.691\pm0.155$&$0.00\pm0.00$&$5.71\pm0.243$&$0.893\pm0.00999$&$5.74\pm0.224$&$0.884\pm0.0105$\\
    &NeuralSD&$\mathbf{5.68\pm0.280}^\ddagger$&$\mathbf{9.59\pm1.22}^\ddagger$&$\mathbf{0.646\pm0.148}^\ddagger$&$0.00\pm0.00$&$\mathbf{5.44\pm0.284}^\ddagger$&$\mathbf{0.912\pm0.0116}^\ddagger$&$\mathbf{5.48\pm0.272}^\ddagger$&$\mathbf{0.904\pm0.0123}^\ddagger$\\\cmidrule(lr){1-1}\cmidrule(lr){2-2}\cmidrule(lr){3-6}\cmidrule(lr){7-8}\cmidrule(lr){9-10}
    
    $200$&RSD&$5.85\pm 0.193$&$9.36\pm 0.944$&$0.610\pm 0.109$&$0.00\pm0.00$&$5.83\pm0.180$&$0.895\pm0.00756$&$5.84\pm0.180$&$0.888\pm0.00847$\\
    &NeuralSD&$\mathbf{5.80 \pm 0.207}^\ddagger$&$\mathbf{8.96 \pm 0.943}^\ddagger$&$\mathbf{0.571\pm 0.106}^\ddagger$&$0.00\pm0.00$&$\mathbf{5.54\pm0.205}^\ddagger$&$\mathbf{0.915\pm0.00821}^\ddagger$&$\mathbf{5.59\pm0.207}^\ddagger$&$\mathbf{0.906\pm0.00920}^\ddagger$\\\bottomrule
    \addlinespace[1ex]
    \multicolumn{10}{r}{$^\dagger p < 0.05\quad ^\ddagger p < 0.01$}
    \end{tabular}
    }
\end{table*}
}
Using TSD as a building block, we propose NeuralSD, a novel neural network architecture that models a parameterized mechanism $f_{\bm{\theta}}$ to learn a matching mechanism from examples while ensuring SP, as motivated in the problem~\eqref{eq:search_over_theta}.
The architecture of NeuralSD is presented in Figure~\ref{fig:neuralsd}.
NeuralSD extends the SD mechanism by treating the agent ranking as a learnable function of public contextual information. 
Concretely, NeuralSD first parametrically computes an agent ranking from agents' public contextual information, and then outputs a matching outcome based on SD and the ranking.
Because TSD is differentiable, we can optimize the agent ranking via backpropagation. 
For the sub-network, we adopt an attention-based \cite{vaswani2017attention} architecture to compute the ranking, ensuring both context-awareness and scale-invariance to the number of agents.
By unifying this component with TSD, we optimize the parameters through the ranking, thereby constructing a learnable SD that forms a parameterized family of SP mechanisms.
Our approach can satisfy SP by this architecture because we determine the ranking solely based on agents' public contextual information that are assumed to be publicly known a priori.
Through this architecture, NeuralSD represents a parameterized family of SP matching mechanisms in the problem~\eqref{eq:search_over_theta}, which can be trained end-to-end from example matchings.

NeuralSD comprises two components: the {\em ranking block}, which takes the contexts $\bm{X}_W$ and $\bm{X}_F$ and outputs a ranking matrix $\tilde{\bm{R}} \in \mathbb{R}^{(n+m) \times (n+m)}$, and the {\em matching block}, which runs TSD to output a matching matrix $\hat{\bm{M}} = {\rm TSD}(\mathbf{P}_W, \mathbf{P}_F, \tilde{\bm{R}})$.

The ranking block $\bm{R}_{\bm{\theta}}$ is an attention-based architecture~\cite{vaswani2017attention}, defined as 
$
    \tilde{\bm{R}} = 
    \bm{R}_{\bm{\theta}}(\bm{X}_W,\bm{X}_F) := \mathrm{SoftSort}_\tau \circ \mathrm{TieBreak} \circ \mathrm{Linear} \circ \mathrm{SelfAttn}(\bm{X}),
$
where $\bm{X} := [\bm{X}_W; \bm{X}_F]$ represents the concatenation of worker contexts $\bm{X}_W$ and firm contexts $\bm{X}_F$. 
The function $\mathrm{SelfAttn}(\bm{X}) := \mathrm{Attention}(\bm{X}\bm{W}^Q, \bm{X}\bm{W}^K, \bm{X}\bm{W}^V)$ applies self-attention \cite{vaswani2017attention} by parameters $\bm{W}^Q, \bm{W}^K, \bm{W}^V\in\R^{d\times d_{\rm emb}}$, where $d$ is the dimension of contexts and $d_{\rm emb}$ is the embedding dimension. 
The linear transformation $\mathrm{Linear}(\bm{A}) := \bm{w}\bm{A}^\top + b$ transforms $\bm{A}$ into a one-dimensional vector by parameters $\bm{w}\in\R^{d_{\rm emb}}$ and $b\in\R$. 
The tie-breaking function $\mathrm{TieBreak}(\bm{a}) := \bm{a} + \mathrm{rank}(\bm{a})$ adds the rank of vector $\bm{a}$, where $\mathrm{rank}(\bm{a}) := [\#\{j\mid \text{$a_j < a_i$ or $a_j = a_i$ and $j < i$} \}]_i$.\footnote{
    Actually, this function is not strictly differentiable with respect to $\bm{a}$.
    Nevertheless, we detached the term $\mathrm{rank}(\bm{a})$ from the computational graph, effectively treating the function as though it simply adds a constant.
} 
Finally, $\mathrm{SoftSort}_\tau(\bm{a})$ is defined as $\mathrm{softmax}\left(\frac{-|\mathrm{sort}(\bm{a})^\top\mathds{1} - \mathds{1}^\top\bm{a}|}{\tau}\right)$, where $|\bm{A}| := [|A_{i,j}|]_{i,j}$ is the element-wise absolute value; $\mathrm{softmax}$ is applied row-wise and $\mathds{1}$ denotes the all-one vector with the same dimension as that of $\bm{a}$.
This provides a differentiable approximation of the $\texttt{argsort}$ operator \cite{prillo2020softsort}.
If there are ties in $\bm{a}$, SoftSort cannot detect them, causing the output $\tilde{\bm{R}}$ to deviate from a permutation matrix.
Therefore, the tie-breaking function is necessary to ensure that the output remains a valid approximation of a permutation matrix.

The matching block takes the estimated ranking $\tilde{\bm{R}}$ to predict a matching $\hat{\bm{M}}$ by executing ${\rm TSD} (\mathbf{P}_W, \mathbf{P}_F, \tilde{\bm{R}})$.
Note that $\tilde{\bm{R}}$ is not guaranteed to be a hard permutation matrix, and consequently, $\hat{\bm{M}}$ may not always be a matching matrix, both due to the use of $\mathrm{SoftSort}$. 
We will discuss this further in Section~\ref{subsec:loss_fun}.
Note that we use TSD during training instead of executing SD directly in this matching block because SD is non-differentiable due to discrete operations and therefore unsuitable for gradient-based learning method.

We can prove that NeuralSD inherits the properties of SD.
\begin{proposition}
    For any parameter $\bm{\theta}$ of the ranking block $\bm{R}_{\bm{\theta}}$, the NeuralSD with $\bm{R}_{\bm{\theta}}$ is both SP and Pareto efficient.
    \label{prop:neuralsd}
\end{proposition}
The proof is provided in Appendix~\ref{apdx:prop_neuralsd}.
Intuitively, the output of NeuralSD is a hard matching matrix during inference, because we compute a hard permutation matrix for $\tilde{\bm{R}}$ by simply applying \texttt{argsort}.
Thus, NeuralSD is identical to SD with the ranking represented in $\tilde{\bm{R}}$.
Note that we relaxed the problem in Equation \eqref{eq:search_over_theta} because NeuralSD does not output a hard matching matrix during training. 
Nevertheless, we adhere to the problem in the sense that we can use NeuralSD as a SP matching mechanism at any time during training by fixing the parameters.

Because SD is not IR, neither is NeuralSD.
Moreover, the same upper bound applies to IRV of NeuralSD as Proposition~\ref{prop:irv_upper_bound_sd} since it holds for SD on any rankings.
\begin{proposition}
    Let $I$ be an instance and $\hat{\bm{M}}_I$ be the predicted matching of NeuralSD on $I$ with truthful reports.
    Then, $\max_{I\in\mathcal{I}}\{irv(\hat{\bm{M}}_I,\succcurlyeq_I)\} \le 1/2$.
    \label{prop:irv_upper_bound_nsd}
\end{proposition}
Although the upper bound is $1/2$, our experiments discussed in Section~\ref{sec:experiments} show that NeuralSD is empirically IR in many cases.

We show the computational complexity of the ranking block as follows.
The proof is provided in Appendix~\ref{apdx:prop_computational_complexity_ranking}.
\begin{proposition}
    Let $n$ and $m$ be the number of workers and firms, respectively.
    Then, the computational complexity of the ranking block is $O((n+m)^2)$.
    \label{prop:computational_complexity_ranking}
\end{proposition}

\subsection{Loss Function}\label{subsec:loss_fun}
We define a loss function for training the NeuralSD model based on cross entropy between predicted and target matchings.
During training, we apply a row-wise softmax to the output matrix $\hat{\bm{M}}$ of NeuralSD to obtain probability distributions over partners for each agent.  
We then compute the average row-wise cross entropy loss between the predicted matching matrix $\hat{\bm{M}}$ and the correct matching matrix $\bm{M}$, up to the $n$-th row, which corresponds to all the workers.
This models the discrepancy function $\mathcal{L}$ in Equation~\eqref{eq:search_over_theta}, defined as:
\begin{align}
    \mathcal{L}(\mathrm{NeuralSD}(\mathbf{P}_{W}, \mathbf{P}_{F}, \bm{X}_{W}, \bm{X}_{F}), \bm{M})\notag\\
    := \frac{1}{n}\sum^{n}_{i=1} \mathcal{L}_{\mathrm{CE}}(\mathrm{softmax}(\hat{\bm{M}}_{i,:}), \bm{M}_{i,:}),
    \label{eq:defD}
\end{align}
where $\mathrm{NeuralSD}(\mathbf{P}_W, \mathbf{P}_{F}, \bm{X}_{W}, \bm{X}_{F})$ models $f_{\bm{\theta}}(I, \succcurlyeq)$ in Equation~\eqref{eq:search_over_theta}, $\mathcal{L}_{\mathrm{CE}}$ denotes the cross-entropy loss, and $\bm{A}_{i,:}$ is the $i$-th row of matrix $\bm{A}$.  
We exclude the last row (corresponding to unmatched agents $\perp$) from the loss calculation, because correct prediction of the other rows is sufficient to ensure alignment with the target matching.

\begin{table}[t]
    \centering
    \caption{
        Recovery rate of optimal rankings.
    }
    \label{tab:recovery_rate}
    \scalebox{0.9}{
    \begin{tabular}{lrrr}
    \toprule
      &\multicolumn{3}{c}{Recovery Rate$\uparrow$ ($^\ddagger p < 0.01$)}\\\cmidrule(lr){2-4}
      Model&DA&EH&MH\\\cmidrule(lr){1-1}\cmidrule(lr){2-4}
      RSD& $0.421\pm 0.0141$&$0.421\pm0.155$&$0.414\pm0.167$\\
      NeuralSD& $\mathbf{0.457\pm 0.0236}^\ddagger$ &$\mathbf{0.465\pm0.205}^\ddagger$&$\mathbf{0.456\pm0.0166}^\ddagger$\\\bottomrule
    \end{tabular}
    }
\end{table}
\begin{figure*}[t]
    \centering
    \begin{minipage}[b]{0.24\linewidth}
        \centering
        \includegraphics[width=\linewidth]{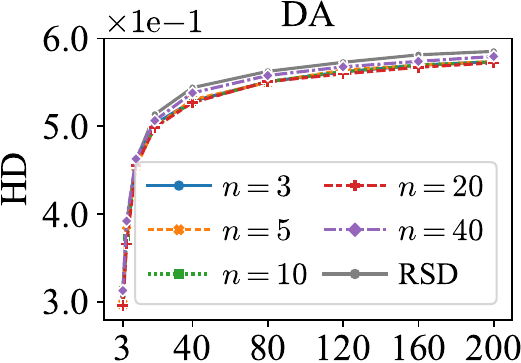}
    \end{minipage}\hfill
    \begin{minipage}[b]{0.24\linewidth}
        \centering
        \includegraphics[width=\linewidth]{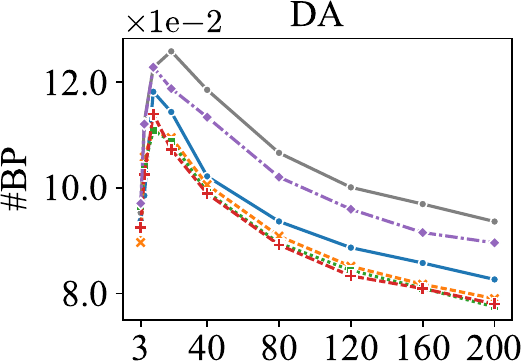}
    \end{minipage}\hfill
    \begin{minipage}[b]{0.24\linewidth}
        \centering
        \includegraphics[width=\linewidth]{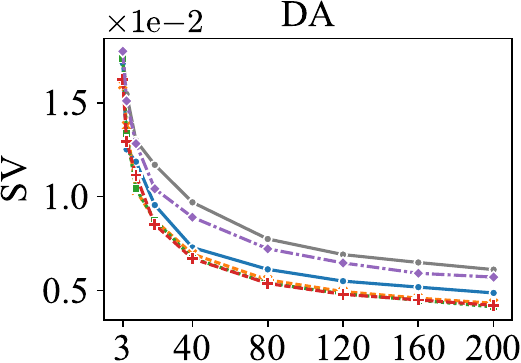}
    \end{minipage}\hfill
    \begin{minipage}[b]{0.24\linewidth}
        \centering
        \includegraphics[width=\linewidth]{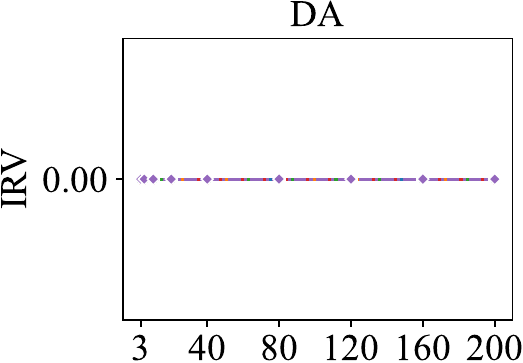}
    \end{minipage}\\
    \begin{minipage}[b]{0.24\linewidth}
        \centering
        \includegraphics[width=\linewidth]{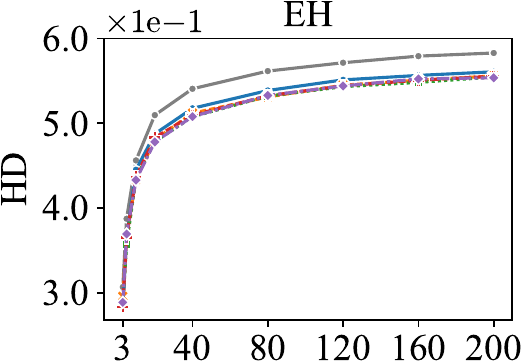}
    \end{minipage}\hfill
    \begin{minipage}[b]{0.24\linewidth}
        \centering
        \includegraphics[width=\linewidth]{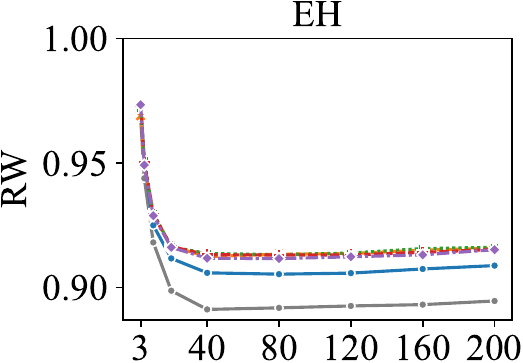}
    \end{minipage}\hfill
    \begin{minipage}[b]{0.24\linewidth}
        \centering
        \includegraphics[width=\linewidth]{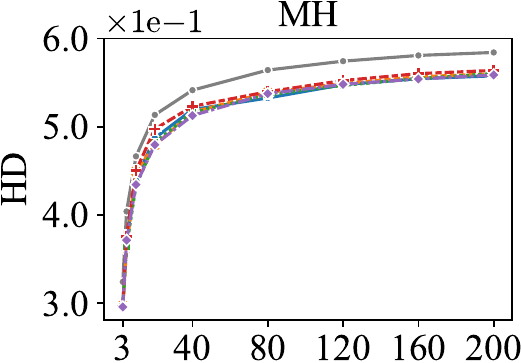}
    \end{minipage}\hfill
    \begin{minipage}[b]{0.24\linewidth}
        \centering
        \includegraphics[width=\linewidth]{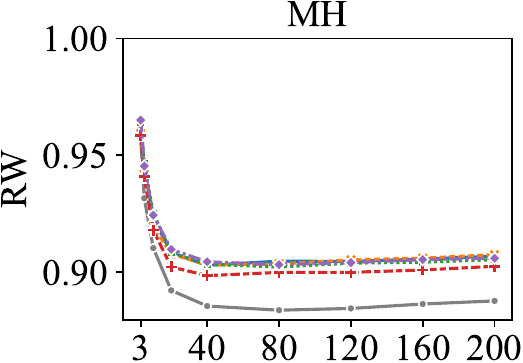}
    \end{minipage}\\
    \caption{
        Average performances across different scales for training and test data. 
        The legend $n=k$ means we set $k$ workers and firms in the training data.
        The horizontal axis represents $n$ in the test data. 
        The vertical axis shows the normalized values of each metric. 
        The legends are consistent across all plots.
    }
    \label{fig:test_cross}
\end{figure*}

\section{Experiments}
\label{sec:experiments}
We experimented to evaluate whether our method can learn a matching mechanism from examples that reflect the implicit goodness of matching outcomes, while satisfying SP.
We further evaluated scale-invariance of our method.
Details of the experimental settings and additional results are provided in Appendix~\ref{apdx:experimental_details}.

{\bf Experimental Setup.}
To simulate example matchings that reflect implicit notions of outcome goodness, we used three existing matching algorithms.
The first is the deferred acceptance (DA)~\cite{gale1962college}, a well-known stable matching algorithm.
The other two are based on the methods proposed by \citet{narasimhan2016automated}: equal-weighted Hungarian matching (EH) and minority-weighted Hungarian matching (MH).  
Both EH and MH define a reward for each worker-firm pair based on the reported preferences, and compute the matching that maximizes the total reward.  
The reward is calculated as the sum of inverted preference ranks ($n+2$ or $m+2$ minus $\mathrm{ord}$, as defined in Section~\ref{subsec:TSD}) each agent assigns to the other; MH further adds bonus weights for certain workers to reflect priority.  
Both EH and MH produce Pareto efficient matchings. Further details are provided in Appendix~\ref{apdx:experimental_details}.
Although these mechanisms are explicitly defined, we treat their outputs as examples of desirable matchings, assuming that the underlying notions of goodness for outcomes, such as stability or rewards, are expressed as examples and not explicitly specified.

We generated public contextual information of agents using random sampling.
We fixed the dimension of the contexts at $d=10$ and generated the worker and firm contexts by sampling independently from $\mathcal{N}(\mathds{1}_{10},\bm{I}_{10})$ and $\mathcal{N}(-\mathds{1}_{10},\bm{I}_{10})$, respectively, where $\bm{I}_{10}$ is the $10$-dimensional identity matrix.
The reports are generated based on Euclidean preferences \cite{bogomolnaia2007euclidean}, where an agent prefers another agent more if the distance between their contexts is shorter.
We generate 1,000 and 750 instances for training and test data, respectively. 
We set $n=m\in\{3, 5, 10, 20, 40, 80, 120, 160, 200\}$.
We conducted training for $n \le 40$ and tests for all values of $n$ to assess the robustness against the increasing number of test agents.

We compared NeuralSD to RSD, which samples agent rankings uniformly at random.
RSD was chosen as the sole baseline because we are only interested in SP mechanisms, and did not compare with other non-SP methods such as DA or neural networks~\cite{ravindranath2023deep,roth1982economics}.

To evaluate how well the learned mechanisms reproduce the example matchings and capture the underlying implicit objectives, we used two types of metrics: one for discrepancy between matchings, and the other for goodness of outcomes. 
As the discrepancy evaluation, we measured the Hamming distance (HD) between the predicted matchings and the example matchings.
For the outcomes, we evaluated the following three metrics for predicted matchings: number of blocking pairs (\#BP), IRV (see Section~\ref{subsec:neuralsd}), and stability violations (SV) when the examples were generated by DA, as DA is stable. 
SV was proposed to measure the extent to which the matchings fail to be stable~\cite{ravindranath2023deep}.
We normalized HD by dividing it by $3n$, \#BP by $n^2$, and SV by $n$.
In Appendix~\ref{apdx:stability_reg}, we report experiments with SV regularized loss function to investigate how to balance the incompatibility between SP and stability.
When the examples were generated by EH or MH mechanisms, we measured the reward (RW) obtained by the predicted matching, normalized by dividing by the optimal reward calculated by the corresponding mechanism.

To further evaluate the learning performance, we examined whether NeuralSD can accurately recover the agent rankings that lead to matchings closest to the target outcomes. 
Specifically, for each test instance with $n=3$, we performed a brute-force search over all possible SD rankings to identify the set of rankings that produce a matching matrix closest to the example matching. 
A predicted ranking $\tilde{\bm{R}}$ was considered correct if it belonged to this set of optimal rankings. 
For RSD, we defined the randomly selected ranking as correct if it happened to fall within the same set. 
We then computed the recovery rate of optimal rankings as the proportion of test instances for which the predicted ranking was included in the optimal set.

{\bf Results and Discussion.}
The average performance is shown in Figure~\ref{fig:test_cross}, and partial results are presented in Table~\ref{tab:results}.
We compared models using the one-sided Wilcoxon signed-rank test on the test instances.
For the test with $n=200$ in Table~\ref{tab:results}, we show the results of our model trained at $n=40$.
The results for other values of $n$ are provided in Appendix~\ref{apdx:experimental_details}.
The results for RSD correspond to how closely RSD performs relative to the mechanisms used for examples.

As the number of agents increased, NeuralSD significantly outperformed RSD in terms of HD as well as metrics for evaluating outcomes, with variances remaining of a comparable order.
These results highlight the advantage of end-to-end learning of agent rankings from agents' public contextual information to obtain better matching outcomes.
Figure~\ref{fig:test_cross} shows that NeuralSD performs well on large-scale instances even when trained on small-scale data, which supports its scale-invariance. 
\#BP did not consistently increase, partially due to the statistical properties of our datasets.
SV did not align with \#BP because it was designed to be zero for stable matchings, without consideration for matchings with a positive number of blocking pairs~\cite{ravindranath2023deep}.
IRV of all models was zero in our experiments, but IRV can be added to the loss function in case we prioritize IR more.

The recovery rate is presented in Table~\ref{tab:recovery_rate}. 
To compare the recovery rate, we trained our proposed model with $n=3$ and tested on 20 different runs and conducted Wilcoxon signed-rank test.
We confirmed that NeuralSD recovers the optimal ranking with significantly larger recovery rates than those with RSD.
This result suggests that NeuralSD can learn the best parameters to recover the desired mechanism at approximately 45\% and can benefit from contextual information for estimating preferred matching outcomes expressed by the examples, rather than using a random ranking as in RSD.

\section{Conclusion and Limitations}
\label{sec:conclusion_limitations}
In this paper, we introduced NeuralSD, a neural network that learns strictly SP, scale-invariant, context aware matching mechanisms from examples.  
Experimental results confirmed that NeuralSD can learn from the example matchings to reflect the implicit notions of goodness for outcomes.
The main limit is the heavy forward pass, so better scaling is left for future work.
The main experimental limitation was the lack of access to real-world datasets, which we plan to address in future work. 
Our work has potential negative impacts due to the drawbacks of SD itself, which can output unstable matchings and would yield matching outcomes that are unsatisfactory for agents by prioritizing some agents. 
A potential solution is to incorporate fairness constraints, which we will consider in future work.



\begin{ack}
This work was supported by JST BOOST, Grant Number JPMJBS2407; JST FOREST Program, Grant Number JPMJFR232S; and JSPS KAKENHI Grant-In-Aid 21H04979.
\end{ack}



\bibliography{mybibfile}

\clearpage
\appendix
\section{Proof of Proposition~\ref{prop:sd}}
\label{apdx:prop_sd}
We prove Proposition~\ref{prop:sd} as follows.
\begin{proof}
Consider an instance with $n$ workers and $m$ firms with truthful preference profile $\succcurlyeq = (\succcurlyeq_{w_1}, \dots, \succcurlyeq_{w_n}, \succcurlyeq_{f_1}, \dots, \succcurlyeq_{f_m})$. 
Fix an arbitrary ranking $\bm{r} = (r_1, \dots, r_{n+m})$ over the agents and consider a matching $\mu$ determined by SD under the ranking $\bm{r}$ and the reports $\succcurlyeq$. 
We say that the agent $r_k$ is {\em active} if, at the $k$-th iteration, the agent is not yet matched and selectively matches with a partner.
Let $r_{(1)} = r_1, r_{(2)}, \dots$ be active agents in order.
Conversely, an agent is {\em passive} if, at the $k$-th iteration, the agent is already matched because their partner is chosen as an agent with a preceding rank.

First, we prove Pareto efficiency.
Suppose another matching $\mu'$ that Pareto dominates $\mu$ under the preference profile $\succcurlyeq$. 
Then, there must exist at least one agent $r_k$ such that $\mu'(r_k) \succ_{r_k} \mu(r_j)$.
We consider two cases:
\begin{enumerate}[(1)]
    \item Suppose that $r_k$ is active.  
    In this case, because the first active agent can choose their most preferred option, it must be that $r_k \neq r_{(1)} = r_1$.  
    If $r_k = r_{(2)}$, then due to $\mu'(r_{(2)}) \neq \mu(r_{(2)})$, we have $\mu'(r_{(2)}) \in \{r_{(1)}, \mu(r_{(1)})\}$.  
    However, in either case, $\mu'(r_{(1)}) \not\succcurlyeq \mu(r_{(2)})$, so it cannot be $r_{(2)}$.  
    Similarly, for any active agent, changing their match does not make $\mu'$ Pareto dominate $\mu$.  
    This contradicts the definition of $\mu$.
    \item Suppose that $r_k$ is passive.
    Then, $\mu$ and $\mu'$ must be identical because all active agents match with the same partner under both matchings, and the behavior of SD uniquely determines the partners of the passively matched agents.
    This contradicts the definition of $\mu'$.
\end{enumerate}

Therefore, $\mu$ is Pareto efficient.

Next, we prove SP.
Suppose an agent $a = r_k$ at an arbitrary rank $k$ is better off by a reported profile $\succcurlyeq_a'\neq \succcurlyeq_a$.
Because $\bm{r}$ is predefined before collecting reports, the list of options $l_a\subset W\cup F$ that $a$ can designate at iteration $k$ is determined irrespective of $\succcurlyeq_a'$.
Because the most preferred option in this $l_a$ is fixed, there is no gain by submitting other preferences.
Therefore, there is no gain from submitting the untruthful $\succcurlyeq_a'$, and thus SD satisfies SP.  
\end{proof}


\section{Description of TSD}
\label{apdx:tensorserialdictatorship}
This section describes the computation of SD through tensor operations. 
The main algorithm is presented in Algorithm~\ref{alg:tsd}, which utilizes the helper function described in Algorithms~\ref{alg:createrankingmasks} and~\ref{alg:findcounterpart}.
We have introduced additional necessary notations before explaining these three algorithms.

\paragraph{Additional Notations.}
$\bm{0}_d$ and $\mathds{1}_d$ denote $d$-dimensional all-zero and all-one row vectors, respectively.
$\bm{O}_{p,q}$ denotes a zero matrix with $p$ rows and $q$ columns.
If not specified, the subscripts $d$, $p$, and $q$ are omitted.
Moreover, we used indexing notation similar to that of Numpy.
For a matrix $\bm{X}\in\R^{p\times q}$, $\bm{X}_{:i,j} := [X_{1,j},\dots, X_{i,j}]^\top$ represents the top $i$ elements of the $j$-th column, and $\bm{X}_{i:,j} := [X_{i,j},\dots,X_{q,j}]^\top$ represents elements from the $i$-th row to the end of the $j$-th column.
Furthermore, $\bm{y}_{:k} := [y_1,\dots, y_k]$ denotes the first $k$ elements of the vector $\bm{y}$.

We use common operations on tensors, as typically found in libraries such as Numpy and PyTorch.
$[\bm{x}\|c] := [x_1, \dots, x_d, c]$ is a row vector obtained by concatenating a vector $\bm{x} = [x_1, \dots, x_d]$ with a scalar $c$.
For a three-dimensional tensor $\mathbf{A}$ and column vector $\bm{x}^\top$, the product $\mathbf{A}\bm{x}^\top$ involves broadcasting: $\mathbf{A}\bm{x}^\top := \sum_i x_i \bm{A}_i$, where $\bm{A}_i$ is the $i$-th matrix within $\mathbf{A}$. 
For two vectors of identical dimensions, $\bm{x} = [x_1, \dots, x_d]$ and $\bm{y} = [y_1, \dots, y_d]$, and their element-wise product is denoted as $\bm{x} \circ \bm{y} := [x_1y_1, \dots, x_dy_d]$. 
With a slight abuse of notation, we define $\mathbf{A} \circ \bm{x} := [x_1\bm{A}_1, \dots, x_d\bm{A}_d]$.
$\texttt{repeat}(\bm{x}, p) := \mathds{1}_p^\top \bm{x}$ is the matrix obtained by repeating $\bm{x}$ $p$ times along the row direction. 
$\texttt{cumsum}(\bm{x}) := [\sum_{i=1}^j x_i]_{j=1, \dots, d}$ is a function that calculates the cumulative sum of a row vector $\bm{x}$. 
For a matrix $\bm{A}$, $\texttt{colsum}(\bm{A})$ produces the row vector obtained by summing columns. 
The function $\mathrm{ReLU}(x) := \max\{x,0\}$ is applied to a scalar $x$, whereas for tensors, it is applied to each individual element.

\paragraph{Algorithm~\ref{alg:createrankingmasks}.}
\begin{algorithm}[tb]
    \caption{\textsc{CreateRankingMasks}}\label{alg:createrankingmasks}
    \textbf{Input}: A ranking matrix $\bm{R}\in\R^{(n+m)\times (n+m)}$.\\
    \textbf{Output}: Ranking masks $\bm{U}_W^{(1)},$ $\dots,$ $\bm{U}_W^{(n+m)},$ $\bm{U}_F^{(1)},$ $\dots,$ $\bm{U}_F^{(n+m)}$.
    \begin{algorithmic}[1]
          \STATE Initialize $\bm{U}_W^{(1)},\dots,\bm{U}_W^{(n+m)}$ by $\bm{O}_{m+1,m+1}$
          \STATE Initialize $\bm{U}_F^{(1)},\dots,\bm{U}_F^{(n+m)}$ by $\bm{O}_{n+1,n+1}$
          \FOR{$k = 1,\dots, n+m$}
            \STATE $\bm{U}_W^{(k)}\leftarrow 
            \sum_{r=n+1}^{n+m}R_{r,k}[\underbrace{\bm{0}^\top\cdots\stackrel{r-n}{-\mathds{1}^\top}\cdots\bm{0}^\top}_{m}\bm{0}^\top]^\top$
            \STATE $\bm{U}_F^{(k)}\leftarrow \sum_{r=1}^n R_{r,k} [\underbrace{\bm{0}^\top\cdots\stackrel{r}{-\mathds{1}^\top}\cdots\bm{0}^\top}_{n}~\bm{0}^\top]^\top$
          \ENDFOR
        \STATE \textbf{return} $\bm{U}_W^{(1)},\dots,\bm{U}_W^{(n+m)}, \bm{U}_F^{(1)},\dots,\bm{U}_F^{(n+m)}$
    \end{algorithmic}
\end{algorithm}
In Algorithm~\ref{alg:createrankingmasks}, we compute \textit{ranking masks}. 
These masks are matrices designed to selectively set zeros on the rows corresponding to matched agents in each iteration of the SD.
We assume the input $\bm{R}$ as a ranking matrix, which is a permutation matrix with size $n+m$ and represents a ranking in SD. 
Then, the matrix $\bm{U}_W^{(k)}$ is computed as follows ($\bm{U}_F^{(k)}$ is defined in a similar manner):
\begin{align}
    \bm{U}_W^{(k)} = 
    \left\{
    \begin{array}{cl} 
        \begin{bmatrix}
            0 & \cdots & 0\\
            \vdots & \ddots & \vdots\\
            -1 & \cdots & -1\\
            \vdots & \ddots & \vdots\\
            0 & \cdots & 0\\
            0 & \cdots & 0\\
        \end{bmatrix}
        \begin{array}{c}
        1\\\vdots\\j-n\\\vdots\\m\\m+1
        \end{array}
    &(\text{if $R_{jk} = 1$})\\
    \bm{O}_{m+1, m+1}
    &(\text{otherwise})
    \end{array}\right.
    .
\end{align}
\paragraph{Algorithm~\ref{alg:findcounterpart}.}
\begin{algorithm}[tb]
   \caption{\textsc{FindCounterpart}}\label{alg:findcounterpart}
   \textbf{Input}: A binary matrix $\bm{P}$.\\
   \textbf{Output}: A vector $\bm{c}^\top$.
    \begin{algorithmic}[1]
        \STATE $\bm{c}\leftarrow \texttt{colsum}(\bm{P})$
        \STATE $\bm{c}\leftarrow \texttt{cumsum}(\bm{c})$
        \STATE $\bm{c}\leftarrow \texttt{triangle\_window}(\bm{c})$
        \STATE $\bm{c}^\top\leftarrow \bm{P}\bm{c}^\top$
        \STATE \textbf{return} $\bm{c}^\top$
    \end{algorithmic}
\end{algorithm}

\begin{algorithm}[tb]
\caption{TSD}\label{alg:tsd}
\textbf{Input}: Preference tensor $\mathbf{P}_W, \mathbf{P}_F$, and ranking matrix $\bm{R}$\\
\textbf{Output}: Matrix representation $\bm{M}$ of the matching result given by SD under $\bm{R}$
\begin{algorithmic}[1] 
\STATE $\bm{U}_W^{(1)},\dots,\bm{U}_W^{(n+m)}, \bm{U}_F^{(1)},\dots,\bm{U}_F^{(n+m)}$\\
$\leftarrow\textsc{CreateRankingMasks}(\bm{R})$
\STATE $\bm{M}\leftarrow \bm{O}_{n+1,m+1}$
\FOR{$k = 1,\dots, n+m$}
	\STATE $\bm{d}_w, \bm{d}_f \leftarrow \bm{R}_{:n,k}^\top, \bm{R}_{n+1:,k}^\top$
            \hfill $\langle 1\rangle$
	\STATE $\bm{P}_w, \bm{P}_f\leftarrow \mathbf{P}_W\bm{d}_w^\top, \mathbf{P}_F\bm{d}_f^\top$
    	\hfill $\langle 2\rangle$
	\STATE $\bm{c}_w^\top \leftarrow \textsc{FindCounterpart}(\bm{P}_w)$
        \STATE $\bm{c}_f^\top \leftarrow \textsc{FindCounterpart}(\bm{P}_f)$\hfill $\langle 3\rangle$
	\STATE $\bm{M}_w, \bm{M}_f \leftarrow [\bm{d}_w\|0]^\top \bm{c}_w, [\bm{d}_f\|0]^\top \bm{c}_f$
    	\hfill $\langle 4\rangle$
	\STATE $\bm{M} \leftarrow \bm{M} + \bm{M}_w + \bm{M}_f^\top$
    	\hfill $\langle 5\rangle$
	\STATE $\bm{V}_W\leftarrow (-1) \cdot \texttt{repeat}(\bm{c}_w\circ [\mathds{1}_m\| 0], m+1)^\top$
        \STATE $\bm{V}_F\leftarrow (-1) \cdot \texttt{repeat}(\bm{c}_f\circ [\mathds{1}_n\| 0], n+1)^\top$
	    \hfill $\langle 6\rangle$
	\STATE $\mathbf{P}_W \leftarrow \mathbf{P}_W + \bm{U}_W^{(k)} + \bm{V}_W$
        \STATE $\mathbf{P}_F\leftarrow \mathbf{P}_F+\bm{U}_F^{(k)}+\bm{V}_F$
		\hfill $\langle 7\rangle$
	\STATE $\mathbf{P}_W\leftarrow\mathrm{ReLU}(\mathbf{P}_W)\circ (\mathds{1}_{n+1}-\bm{c}_f)_{:n}^\top$
        \STATE $\mathbf{P}_F\leftarrow  \mathrm{ReLU}(\mathbf{P}_F)\circ (\mathds{1}_{m+1}-\bm{c}_w)_{:m}^\top$
    	\phantom{\hspace{4.65em}}$\langle 8\rangle$
\ENDFOR
\STATE \textbf{return} $\bm{M}$
\end{algorithmic}
\end{algorithm}
The function $\texttt{triangle\_window}(x)$ in Algorithm~\ref{alg:findcounterpart} represents a type of window function defined as follows:
\begin{align}
    \texttt{triangle\_window}(x) :=
    \left\{\begin{array}{cl}
        0 & \text{(if $x \le 0$)}\\
        x & \text{(if $0 < x \le 1$)}\\
        2-x & \text{(if $1 < x \le 2$)}\\
        0 & \text{(if $x > 2$)}
    \end{array}\right..\label{eq:deftw}
\end{align}
Given a binary matrix $\bm{P}$, Algorithm~\ref{alg:findcounterpart} is designed to identify the leftmost column containing at least one `$1$'. 
Formally, we can give the following result:

\begin{proposition}
    If the input $\bm{P}$ satisfies the conditions
    \begin{inparaenum}[(1)]
        \item $\exists p, q \in \mathbb{N}, \bm{P} \in \{0,1\}^{p \times q}$ and
        \item all columns and all rows contain at most one `$1$',
    \end{inparaenum}
    then $\textsc{FindCounterpart}(\bm{P})$ outputs $\bm{0}_p$ when $\bm{P} = \bm{O}_{p,q}$, or it outputs the leftmost column that contains at least one `$1$'.\label{prop:findcounterpart}
\end{proposition}
\begin{proof}
    The statement clearly holds when $\bm{P} = \bm{O}_{p,q}$.
    
    Let $\bm{P}\neq \bm{O}_{p,q}$.
    We define the column indices containing at least one `$1$', ordered from left to right, as $k_1, \dots, k_l$.
    Then, $\texttt{colsum}(\bm{P}) = \sum^l_{j=1}\bm{e}_{k_j},$ where $\bm{e}_i$ is a one-hot vector with the $i$-th element equal to $1$. 
    Therefore, we have
    \begin{align*}
        &\texttt{cumsum}(\texttt{colsum}(\bm{P}))\\
        &= \lbrack 0,\dots, \stackrel{k_1-1}{0}, \stackrel{k_1}{1},\dots, \stackrel{k_2-1}{1}, \stackrel{k_2}{2},\dots, \stackrel{q}{l}\rbrack.
    \end{align*}
    From Equation~\ref{eq:deftw}, which defines \texttt{triangle\_window}, we have
    \begin{align*}
        &\texttt{triangle\_window}(\texttt{cumsum}(\texttt{colsum}(\bm{P})))\\
        &= \lbrack 0,\dots, \stackrel{k_1-1}{0}, \stackrel{k_1}{1},\dots, \stackrel{k_2-1}{1}, \stackrel{k_2}{0},\dots, \stackrel{q}{0}\rbrack\\
        &= \sum_{j=k_1}^{k_2-1} \bm{e}_j.
    \end{align*}
    Consequently, the output $\bm{c}^\top$ is
    \begin{align}
        \bm{P}\sum^{k_2-1}_{j=k_1} \bm{e}_j^\top = \sum^{k_2-1}_{j=k_1} \bm{p}_j^\top = \bm{p}_{k_1}^\top, 
    \end{align}
    where $\bm{p}_j^\top$ is the $j$-th column vector of $\bm{P}$.
    The final equality holds because the $(k_1+1)$-th to $(k_2-1)$-th columns are all zero.
\end{proof}

\paragraph{Algorithm~\ref{alg:tsd}.}
\begin{figure}[t]
    \centering
    \includegraphics[width=1.0\linewidth]{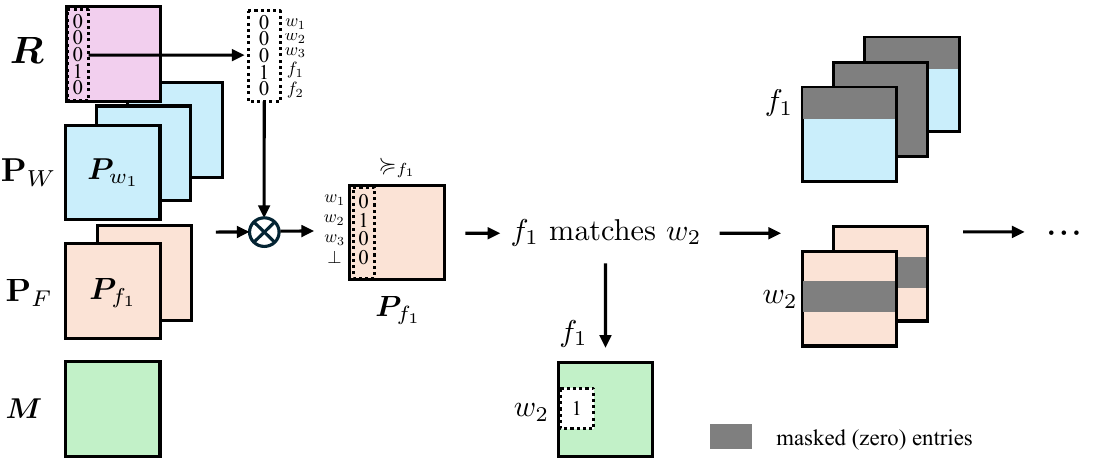}
    \caption{Conceptual visualization of TSD: SD computation using tensor operations.}
    \label{fig:tsd}
\end{figure}
Using Algorithms~\ref{alg:createrankingmasks} and~\ref{alg:findcounterpart}, TSD in Algorithm~\ref{alg:tsd} computes the exact result of SD.
We prove Proposition~\ref{prop:tsd} as follows.
\begin{proof}
    When $r_k = x$ for an agent $x \in W \cup F$, we refer to the agent as the $k$-th active agent (see Appendix~\ref{apdx:prop_sd}). 
    When $r_k = x$ and $x$ selectively matches with an agent $y\neq\perp$, $y$ is said to be {\em passively matched} with $x$.
    For notational convenience, we omit subscripts when denoting the dimensions of zero vectors or zero matrices.
    
    Regarding the main loop starting at Line 3, we propose the following loop invariants based on the loop variable $k$:
    \begin{enumerate}[(A)]
        \item $\forall i \in [n], \bm{P}_{w_i}$ satisfies the two conditions stated in Proposition~\ref{prop:findcounterpart}.
        \item $\forall j \in [m], \bm{P}_{f_j}$ satisfies the two conditions stated in Proposition~\ref{prop:findcounterpart}.
        \item For all agents $x \in W \cup F$, $\bm{P}_x = \bm{O}$ if and only if there exists $r \in [k-1]$ such that $x$ is passively matched by the $r$-th active agent.
        \item For all agents $x \in W \cup F$, the rows corresponding to any agents that are matched in any iteration from $1$ to $k-1$ are all zero, and the other rows remain unchanged from the input.
        \item The tentative result of SD up to the $(k-1)$-th iteration is represented by $\bm{M}$.
    \end{enumerate}
   The loop invariants (A) and (B) follow from (D), but we state them explicitly for the convenience of our proof.
   Initially, before the loop begins, these conditions are clearly met.
   
   Next, we consider the $k$-th iteration assuming that the loop invariants have been maintained up to the $(k-1)$-th iteration.
   Without loss of generality, we assume that $r_k = w_i$ for some $i$. 
   Let $\overline{f}_{i_j} \in \overline{F}$ denote the $j$-th most preferred firm in $\succcurlyeq_{w_i}$. 
   We define $l$ as the rank of the option in $\overline{F}$ that is most preferred in $\succcurlyeq_{w_i}$ among those unmatched until the $(k-1)$-th iteration.

    The internal procedure before and after step $\langle 5\rangle$ is declined below.

    \paragraph{[$\langle 1\rangle$ to $\langle 5\rangle$]}
    Because $\bm{d}_f = \bm{0}$, it follows that $\bm{P}_f = \bm{O}, \bm{c}_f=\bm{0}$, and by Proposition~\ref{prop:findcounterpart}, $\bm{M}_f = \bm{O}$.
    Additionally, $\bm{d}_w = \bm{e}_i$ (a one-hot vector with its $i$-th element being $1$); hence, $\bm{P}_w = \bm{P}_{w_i}$, as indicated in step $\langle 2\rangle$.
    
    Consider the scenario where $w_i$ is passively matched by the $(k-1)$-th iteration. 
    Then, $\bm{P}_w = \bm{O}$ due to the loop invariant (C). 
    Consequently, it follows that $\bm{c}_w = \bm{0}$ and $\bm{M}_w = \bm{O}$. 
    Thus, $\bm{M}$ remains unchanged, and because $w_i$ has already been matched, the loop invariant (E) is maintained.

    Now, consider the case where $w_i$ is not passively matched by the $k$-th iteration.
    By the loop invariant (C), we have $\bm{P}_w \neq \bm{O}$.
    Moreover, the $l$-th column is the leftmost one in $\bm{P}_w$ containing at least one `1'. 
    This assertion holds because if there were another column, $l'$, to the left of $l$, then two conditions could apply: if $\overline{f}_{l'} \neq \perp$ and is matched by the $(k-1)$-th iteration, then the row corresponding to $\overline{f}_{l'}$ would be zero (due to loop invariant (D)), leading to a contradiction; otherwise, it contradicts the definition of $l$. 
    Therefore, from the loop invariant (A) and Proposition~\ref{prop:findcounterpart}, $\bm{c}_w = \bm{e}_{i_l}$, and thus, $\bm{M}_w = [\bm{e}_i\| 0]^\top\bm{e}_{i_l} = \bm{E}_{i,i_l}$ (the matrix with $1$ at the $i$-th row and $i_l$-th column and $0$ elsewhere).
    By updating $\bm{M}$ to $\bm{M} + \bm{M}_w = \bm{M} + \bm{E}_{i,i_l}$, the loop invariant (E) is preserved.
    
    \paragraph{[$\langle 6\rangle$ to $\langle 8\rangle$]}
    Given that $\bm{c}_f = \bm{0}$, it follows that $\bm{V}_F = \bm{O}$. 
    Additionally, $\bm{U}_W^{(k)} = \bm{O}$ because $r_k=w_i$.

    Consider the case where $w_i$ is passively matched by the $(k-1)$-th iteration. 
    As previously stated, $\bm{P}_w = \bm{O}$ and $\bm{c}_w = \bm{0}$, which results in $\bm{V}_W = \bm{O}$. 
    Therefore, $\mathbf{P}_W$ remains unchanged after step $\langle 8\rangle$. 
    In addition, $\mathbf{P}_F$ also remains unchanged because the $i$-th row in all firms' preference matrices is already zero, as per the loop invariant (D). 
    Consequently, the loop invariants (A), (B), (C), and (D) are maintained after this iteration.

    Now, consider the scenario where $w_i$ is not passively matched by the $k$-th iteration. 
    In this case, $w_i$ matches with $\overline{f}_{i_l}$. 
    If $\overline{f}_{i_l} = \perp$, then $\mathbf{P}_W$ remains unchanged because $\bm{c}_w = \bm{e}_{m+1}$; thus, $\bm{V}_W = (-1)\cdot\texttt{repeat}(\bm{e}_{m+1}\circ[\mathds{1}_m\| 0], m+1) = \bm{O}$. 
    Similarly, $\mathbf{P}_F$ remains unchanged because $(\mathds{1}_{m+1} - \bm{c}_w){:m} = \bm{0}$. 
    As a result, the loop invariants (A), (B), (C), and (D) are maintained after this loop.
    
    Consider the situation where $\overline{f}_{i_l} \neq \perp$. 
    In this case, $\bm{V}_W = [\bm{0}^\top\cdots\stackrel{i_l}{-\mathds{1}^\top}\cdots\bm{0}^\top]^\top$, causing the $i_l$-th row in all workers' preference matrices to become zero after step $\langle 8\rangle$. Additionally, because $\bm{U}_F^{(k)} = [\bm{0}^\top\cdots\stackrel{i}{-\mathds{1}^\top}\cdots\bm{0}^\top]^\top$, the $i$-th row in all firms' preference matrices becomes zero after step $\langle 8\rangle$. The preference matrix of $\overline{f}_{i_l}$ is set to $\bm{O}$ by step $\langle 8\rangle$. Hence, the loop invariants (C) and (D), and consequently (A) and (B), are maintained after this loop.

    After the loop terminates, the desired conclusion follows directly from the loop invariant (E).
\end{proof}

\paragraph{Computational Cost.}
In Algorithm~\ref{alg:createrankingmasks}, the computation of $\bm{U}_W^{(k)}$ requires $O((n+m)(m+1)^2) = O((n+m)m^2)$, and the computation of $\bm{U}_F^{(k)}$ requires $O(n(n+1)^2) = O(n^3)$. Therefore, the entire algorithm incurs a time complexity of $O((n+m)\{(n+m)m^2 + n^3\})$.

In Algorithm~\ref{alg:tsd}, the main bottleneck is the update of $\mathbf{P}_W$ and $\mathbf{P}_F$ in step $\langle 7\rangle$, which requires $O(n(m+1)^2 + ((n+1)^2m)) = O(nm^2 + n^2m)$. Hence, the main loop requires $O((n+m)(nm^2 + n^2m))$, and consequently, this incurs a time complexity of $O((n+m)\{(n+m)m^2 + n^3\})$ from Algorithm~\ref{alg:createrankingmasks}.
Assuming $n \geq m$, this time complexity simplifies to $O(n^4)$.


\section{Proof of Proposition~\ref{prop:neuralsd}}
\label{apdx:prop_neuralsd}
The proof of Proposition~\ref{prop:neuralsd} is as follows.
\begin{proof}
  We show that NeuralSD becomes identical to SD when used following the scenario in Section~\ref{sec:background}. 
  Consider an instance $I$. 
  We compute $\tilde{\bm{R}}$ as a hard permutation matrix from contexts $\bm{X}_W$ and $\bm{X}_F$ by applying the $\texttt{argsort}$ operator.
  Let $\mathrm{NeuralSD}(\cdot, \cdot, \bm{X}_W, \bm{X}_F)$ denote the partially inputted NeuralSD with fixed $\tilde{\bm{R}}$. 
  By its construction, this is identical to $\mathrm{TSD}(\cdot, \cdot, \tilde{\bm{R}})$. 
  From Proposition~\ref{prop:tsd}, this function is identical to SD with the ranking represented in $\tilde{\bm{R}}$. 
  Therefore, using $\mathrm{NeuralSD}(\cdot, \cdot, \bm{X}_W, \bm{X}_F)$ to determine a matching from reports is equivalent to executing SD with the ranking. 
  Consequently, we maintain the SP and Pareto efficiency of NeuralSD, both inherited from SD, as in Proposition~\ref{prop:sd}.
\end{proof}

\section{Experimental Details}
\label{apdx:experimental_details}
\begin{table*}[t]
    \centering
    \caption{
        Required computational resources for experiments.
        For individual computation, the `train' values correspond to one run over 1,000 training instances, and `test' values correspond to one run over 750 test instances.
    }
    \label{tab:compute_resources}
    \begin{tabular}{cc}
    \toprule
       Compute workers & CPU\\\cmidrule(lr){1-2}
       Memory& 16GB (train/test on $n\in \{3,5,10,20,40\}$)\\
       &2TB (test on $n\in \{80, 120, 160, 200\}$)\\\cmidrule(lr){1-2}
       Individual Computation& $<$ 5 min (train/test on $n\in \{3,5,10,20\}$)\\
       & 30 min (train/test on $n\in \{40, 80, 120, 160\}$)\\
       & 90 min (test on $n = 200$)\\\bottomrule
    \end{tabular}
\end{table*}

\begin{table*}[t]
    \centering
    \caption{
        Experimental results on different scales.
        Notations are idencical to those used in Table~\ref{tab:results}.
    }
    \label{tab:results_full}
    \scalebox{0.8}{
    \begin{tabular}{clrrrrrrrr}
    \toprule
    &&\multicolumn{4}{c}{Deferred Acceptance (DA)}&\multicolumn{2}{c}{Equal-weighted Hungarian (EH)}&\multicolumn{2}{c}{Minority-weighted Hungarian (MH)}\\\cmidrule(lr){3-6}\cmidrule(lr){7-8}\cmidrule(lr){9-10}
    $n$&Model&HD$\downarrow(\times 10^{-1})$&\#BP$\downarrow(\times 10^{-2})$&SV$\downarrow(\times 10^{-2})$&IRV$\downarrow$&HD$\downarrow(\times 10^{-1})$&RW$\uparrow$&HD$\downarrow(\times 10^{-1})$&RW$\uparrow$\\\cmidrule(lr){1-1}\cmidrule(lr){2-2}\cmidrule(lr){3-6}\cmidrule(lr){7-8}\cmidrule(lr){9-10}

    $3$&RSD&$3.05\pm 2.82$&$9.53\pm 10.0$&$1.76 \pm 2.16$&$0.00\pm0.00$&$3.07\pm2.80$&$0.968\pm0.0467$&$3.24\pm2.81$&$0.957\pm0.0547$\\
    &NeuralSD&$\mathbf{2.94\pm2.74}$&$\mathbf{9.38\pm10.0}$&$\mathbf{1.71\pm2.19}$&$0.00\pm0.00$&$\mathbf{2.84\pm2.79}$&$\mathbf{0.974\pm0.0441}^\ddagger$&$\mathbf{2.98\pm2.82}^\dagger$&$\mathbf{0.965\pm0.0509}^\ddagger$\\\cmidrule(lr){1-1}\cmidrule(lr){2-2}\cmidrule(lr){3-6}\cmidrule(lr){7-8}\cmidrule(lr){9-10}

    $5$&RSD&$3.90\pm 1.97$&$11.2\pm 7.76$&$1.55\pm 1.46$&$0.00\pm0.00$&$3.87\pm1.93$&$0.944\pm0.0471$&$4.04\pm1.94$&$0.932\pm0.0520$\\
    &NeuralSD&$\mathbf{3.81\pm2.03}$&$\mathbf{10.6\pm7.83}^\dagger$&$\mathbf{1.38\pm1.41}^\ddagger$&$0.00\pm0.00$&$\mathbf{3.67\pm1.98}^\dagger$&$\mathbf{0.950\pm0.0468}^\ddagger$&$\mathbf{3.63\pm1.99}^\ddagger$&$\mathbf{0.945\pm0.0495}^\ddagger$\\\cmidrule(lr){1-1}\cmidrule(lr){2-2}\cmidrule(lr){3-6}\cmidrule(lr){7-8}\cmidrule(lr){9-10}

    $10$&RSD &$4.61\pm 1.23$&$12.3\pm 5.18$&$1.30 \pm 0.815$&$0.00\pm0.00$&$4.56\pm1.17$&$0.918\pm0.0359$&$4.66\pm1.20$&$0.910\pm0.0410$\\
    &NeuralSD &$\mathbf{4.57\pm 1.25}$&$\mathbf{11.1\pm 5.48}^\ddagger$&$\mathbf{1.04\pm 0.813}^\ddagger$&$0.00\pm0.00$&$\mathbf{4.35\pm1.23}^\ddagger$&$\mathbf{0.930\pm0.0372}^\ddagger$&$\mathbf{4.33\pm1.29}^\ddagger$&$\mathbf{0.925\pm0.0418}^\ddagger$\\\cmidrule(lr){1-1}\cmidrule(lr){2-2}\cmidrule(lr){3-6}\cmidrule(lr){7-8}\cmidrule(lr){9-10}

    $20$&RSD&$5.14\pm 0.777$&$12.6\pm 3.67$&$1.17 \pm0.560$&$0.00\pm0.00$&$5.10\pm0.741$&$0.899\pm0.0277$&$5.13\pm0.745$&$0.892\pm0.0307$\\
    &NeuralSD&$\mathbf{4.99\pm0.814}^\ddagger$&$\mathbf{10.7\pm3.56}^\ddagger$&$\mathbf{0.851\pm0.489}^\ddagger$&$0.00\pm0.00$&$\mathbf{4.83\pm0.806}^\ddagger$&$\mathbf{0.916\pm0.0278}^\ddagger$&$\mathbf{4.97\pm0.788}^\ddagger$&$\mathbf{0.902\pm0.0307}^\ddagger$\\\cmidrule(lr){1-1}\cmidrule(lr){2-2}\cmidrule(lr){3-6}\cmidrule(lr){7-8}\cmidrule(lr){9-10}

    $40$&RSD&$5.44\pm 0.474$&$11.9\pm 2.41$&$0.969\pm 0.338$&$0.00\pm0.00$&$5.41\pm0.485$&$0.891\pm0.0188$&$5.41\pm0.481$&$0.885\pm0.0213$\\
    &NeuralSD&$\mathbf{5.38\pm 0.509}^\ddagger$&$\mathbf{11.3\pm 2.40}^\ddagger$&$\mathbf{0.890\pm 0.329}^\ddagger$&$0.00\pm0.00$&$\mathbf{5.08\pm0.522}^\ddagger$&$\mathbf{0.912\pm0.0197}^\ddagger$&$\mathbf{5.13\pm0.509}^\ddagger$&$\mathbf{0.904\pm0.0217}^\ddagger$\\\cmidrule(lr){1-1}\cmidrule(lr){2-2}\cmidrule(lr){3-6}\cmidrule(lr){7-8}\cmidrule(lr){9-10}

    $80$&RSD&$5.63\pm0.343$&$10.7\pm1.63$&$0.774\pm0.205$&$0.00\pm0.00$&$5.62\pm0.311$&$0.892\pm0.0130$&$5.64\pm0.322$&$0.884\pm0.0148$\\
    &NeuralSD&$\mathbf{5.58\pm0.335}^\ddagger$&$\mathbf{10.2\pm1.57}^\ddagger$&$\mathbf{0.722\pm0.199}^\ddagger$&$0.00\pm0.00$&$\mathbf{5.33\pm0.343}^\ddagger$&$\mathbf{0.912\pm0.0140}^\ddagger$&$\mathbf{5.37\pm0.353}^\ddagger$&$\mathbf{0.903\pm0.0156}^\ddagger$\\\cmidrule(lr){1-1}\cmidrule(lr){2-2}\cmidrule(lr){3-6}\cmidrule(lr){7-8}\cmidrule(lr){9-10}
    $120$&RSD&$5.73\pm0.258$&$10.0\pm1.27$&$0.691\pm0.155$&$0.00\pm0.00$&$5.71\pm0.243$&$0.893\pm0.00999$&$5.74\pm0.224$&$0.884\pm0.0105$\\
    &NeuralSD&$\mathbf{5.68\pm0.280}^\ddagger$&$\mathbf{9.59\pm1.22}^\ddagger$&$\mathbf{0.646\pm0.148}^\ddagger$&$0.00\pm0.00$&$\mathbf{5.44\pm0.284}^\ddagger$&$\mathbf{0.912\pm0.0116}^\ddagger$&$\mathbf{5.48\pm0.272}^\ddagger$&$\mathbf{0.904\pm0.0123}^\ddagger$\\\cmidrule(lr){1-1}\cmidrule(lr){2-2}\cmidrule(lr){3-6}\cmidrule(lr){7-8}\cmidrule(lr){9-10}

    $160$&RSD&$5.81\pm0.216$&$9.69\pm1.07$&$0.649\pm0.126$&$0.00\pm0.00$&$5.79\pm0.199$&$0.893\pm0.00854$&$5.81\pm0.203$&$0.886\pm0.00966$\\
    &NeuralSD&$\mathbf{5.74\pm0.222}^\ddagger$&$\mathbf{9.15\pm1.08}^\ddagger$&$\mathbf{0.591\pm0.125}^\ddagger$&$0.00\pm0.00$&$\mathbf{5.52\pm0.240}^\ddagger$&$\mathbf{0.913\pm0.00957}^\ddagger$&$\mathbf{5.54\pm0.237}^\ddagger$&$\mathbf{0.905\pm0.0108}^\ddagger$\\\cmidrule(lr){1-1}\cmidrule(lr){2-2}\cmidrule(lr){3-6}\cmidrule(lr){7-8}\cmidrule(lr){9-10}
    
    $200$&RSD&$5.85\pm 0.193$&$9.36\pm 0.944$&$0.610\pm 0.109$&$0.00\pm0.00$&$5.83\pm0.180$&$0.895\pm0.00756$&$5.84\pm0.180$&$0.888\pm0.00847$\\
    &NeuralSD&$\mathbf{5.80 \pm 0.207}^\ddagger$&$\mathbf{8.96 \pm 0.943}^\ddagger$&$\mathbf{0.571\pm 0.106}^\ddagger$&$0.00\pm0.00$&$\mathbf{5.54\pm0.205}^\ddagger$&$\mathbf{0.915\pm0.00821}^\ddagger$&$\mathbf{5.59\pm0.207}^\ddagger$&$\mathbf{0.906\pm0.00920}^\ddagger$\\\bottomrule
    \addlinespace[1ex]
    \multicolumn{10}{r}{$^\dagger p < 0.05\quad ^\ddagger p < 0.01$}
    \end{tabular}
    }
\end{table*}

\begin{table*}[t]
    \centering
    \caption{
        Evaluation metrics for DA with stability regularization.
    }
    \label{tab:results_lam}
    \begin{tabular}{clrrrr}
    \toprule
    &&\multicolumn{4}{c}{Deferred Acceptance (DA)}\\\cmidrule(lr){3-6}
    $n$&Model&HD $\downarrow(\times 10^{-1})$&\#BP $\downarrow(\times 10^{-2})$&SV $\downarrow(\times 10^{-2})$&IRV$\downarrow$\\\cmidrule(lr){1-1}\cmidrule(lr){2-2}\cmidrule(lr){3-6}
    
    $3$&RSD&$3.05\pm 2.82$&$9.53\pm 10.0$&$1.76 \pm 2.16$&$0.00\pm0.00$\\
    &NeuralSD&$\mathbf{3.01\pm 2.79}$&$\mathbf{9.26\pm 9.79}$&$\mathbf{1.69 \pm 2.15}$&$0.00\pm0.00$\\\cmidrule(lr){1-1}\cmidrule(lr){2-2}\cmidrule(lr){3-6}

    $5$&RSD&$3.90\pm 1.97$&$11.2\pm 7.76$&$1.55\pm 1.46$&$0.00\pm0.00$\\
    &NeuralSD&$\mathbf{3.82\pm 2.03}$&$\mathbf{10.7\pm 7.90}$&$\mathbf{1.37\pm 1.43}^\ddagger$&$0.00\pm0.00$\\\cmidrule(lr){1-1}\cmidrule(lr){2-2}\cmidrule(lr){3-6}

    $10$&RSD&$4.61\pm 1.23$&$12.3\pm 5.18$&$1.30 \pm 0.815$&$0.00\pm0.00$\\
    &NeuralSD&$\mathbf{4.48\pm 1.38}^\ddagger$&$\mathbf{11.2\pm 5.67}^\ddagger$&$\mathbf{1.07\pm 0.881}^\ddagger$&$0.00\pm0.00$\\\cmidrule(lr){1-1}\cmidrule(lr){2-2}\cmidrule(lr){3-6}

    $20$&RSD&$5.14\pm 0.777$&$12.6\pm 3.67$&$1.17 \pm 0.560$&$0.00\pm0.00$\\
    &NeuralSD&$\mathbf{5.00\pm 0.797}^\ddagger$&$\mathbf{11.2\pm 3.55}^\ddagger$&$\mathbf{0.920\pm 0.507}^\ddagger$&$0.00\pm0.00$\\\cmidrule(lr){1-1}\cmidrule(lr){2-2}\cmidrule(lr){3-6}

    $40$&RSD&$5.44\pm 0.474$&$11.9\pm 2.41$&$0.969\pm 0.338$&$0.00\pm0.00$\\
    &NeuralSD&$\mathbf{5.30\pm 0.502}^\ddagger$&$\mathbf{10.5\pm 2.44}^\ddagger$&$\mathbf{0.777\pm 0.315}^\ddagger$&$0.00\pm0.00$\\\midrule

    $80$&RSD&$5.63\pm 0.343$&$10.7\pm 1.63$&$0.774 \pm 0.205$&$0.00\pm0.00$\\
    &NeuralSD&$\mathbf{5.54\pm 0.346}^\ddagger$&$\mathbf{9.67\pm 1.69}^\ddagger$&$\mathbf{0.647 \pm 0.208}^\ddagger$&$0.00\pm0.00$\\\cmidrule(lr){1-1}\cmidrule(lr){2-2}\cmidrule(lr){3-6}

    $120$&RSD&$5.73\pm 0.258$&$10.0\pm 1.27$&$0.691 \pm 0.155$&$0.00\pm0.00$\\
    &NeuralSD&$\mathbf{5.65\pm 0.276}^\ddagger$&$\mathbf{9.14\pm 1.29}^\ddagger$&$\mathbf{0.580 \pm 0.152}^\ddagger$&$0.00\pm0.00$\\\cmidrule(lr){1-1}\cmidrule(lr){2-2}\cmidrule(lr){3-6}

    $160$&RSD&$5.81\pm 0.216$&$9.69\pm 1.07$&$0.649 \pm 0.126$&$0.00\pm0.00$\\
    &NeuralSD&$\mathbf{5.71\pm 0.227}^\ddagger$&$\mathbf{8.63\pm 1.10}^\ddagger$&$\mathbf{0.526 \pm 0.122}^\ddagger$&$0.00\pm0.00$\\\cmidrule(lr){1-1}\cmidrule(lr){2-2}\cmidrule(lr){3-6}
    
    $200$&RSD&$5.85\pm 0.193$&$9.36\pm 0.944$&$0.610\pm 0.109$&$0.00\pm0.00$\\
    &NeuralSD&$\mathbf{5.75 \pm 0.203}^\ddagger$&$\mathbf{8.42 \pm 0.984}^\ddagger$&$\mathbf{0.504\pm 0.109}^\ddagger$&$0.00\pm0.00$\\\bottomrule
    \addlinespace[1ex]
    \multicolumn{6}{r}{$^\ddagger p < 0.01$}
    \end{tabular}
\end{table*}

\begin{table}[t]
    \centering
    \caption{
        Recovery rate of optimal rankings with stability regularization.
    }
    \label{tab:recovery_rate_lam}
    \scalebox{0.9}{
    \begin{tabular}{lr}
    \toprule
      Model&  Recovery Rate$\uparrow$ ($^\ddagger p < 0.01$)\\\cmidrule(lr){1-1}\cmidrule(lr){2-2}
      RSD (baseline)& $0.430\pm 0.0204$\\
      NeuralSD (ours)& $\mathbf{0.451\pm 0.0195}^\ddagger$ \\\bottomrule
    \end{tabular}
    }
\end{table}

\subsection{Experimental Setup}
\paragraph{Matching Mechanisms.}
Consider a reported profile $\succcurlyeq$ with $n$ workers and $m$ firms.
The two variants of the Hungarian rules determine a matching result by solving the following optimization problem:
\begin{align}
    \max_{\bm{M}} &\sum_{i=1}^{n+1}\sum_{j=1}^{m+1} r^h(i,j)M_{i,j},\label{eq:reward}\\
    \mathrm{s.t.,} & \sum_{j=1}^{m+1} M_{i,j} = 1,\quad i = 1,\dots, n,\notag\\
                   & \sum_{i=1}^{n+1} M_{i,j} = 1,\quad j = 1,\dots, m,\notag\\
                   & M_{n+1,m+1} = 0,\notag\\
                   & M_{i,j}\in \{0,1\},\quad i=1,\dots,n+1;j=1,\dots,m+1,\notag
\end{align}
where $h\in\{\mathrm{EH},\mathrm{MH}\}$ denotes the type of matching mechanism, and $r^h(i,j)$ denotes the reward for matching $w_i$ and $f_j$.
We define $r^h(i,j)$ as follows:
\begin{align*}
    r^h(i,j) := \alpha^h(i)\mathrm{invord}(\overline{f}_j,\succcurlyeq_{\overline{w}_i}) + \mathrm{invord}(\overline{w}_i,\succcurlyeq_{\overline{f}_j}),
\end{align*}
where $\overline{w}_i := w_i$ for $i\in [n]$ and $\overline{w}_{n+1} := \perp$.
$\overline{f}_j$ is defined in the same way.
$\mathrm{invord}(\overline{f}, \succcurlyeq_w) = m+2 - \mathrm{ord}(\overline{f},\succcurlyeq_w) = \#\{\overline{f'}\in\overline{F}\mid \overline{f}\succcurlyeq_w\overline{f'}\}$ is the inverted order, and we define $\mathrm{ord}(\cdot,\succcurlyeq_{\overline{w}_{n+1}}) := 0$.
We define the inverted order for $\succcurlyeq_f$ in the same way.
$\alpha^h(i)$ represents the additional weights for $w_i$.
EH gives equal weights to all the workers: $\alpha^\mathrm{EH}(i) \equiv 1 (\forall i\in [n])$, and MH gives higher weights to $1/3^\mathrm{rd}$ of randomly chosen workers: $\alpha^\mathrm{MH}(i) := 2$ if $i$ is among the selected workers, and $\alpha^\mathrm{MH}(i) := 1$ otherwise.
For each train and test instance, the random selection of workers is performed individually.
We solved the optimization problem using the Gurobi optimizer~\cite{gurobi}.

EH and MH are both Pareto efficient.
This is because, if a Pareto improvement exists for the obtained result, the inverted order of one agent will strictly increase, while the inverted order of the other agent will either remain unchanged or increase. 
As a result, the objective function in Equation~\eqref{eq:reward} will strictly increase, which contradicts optimality.

\paragraph{Definition of Preferences.}
For an instance $\langle W, F, \bm{X}_W, \bm{X}_F \rangle$, we define the preferences of both workers and firms as follows:
For a worker $w \in W$ (and similarly defined for every firm), we define $\succcurlyeq_w$ such that for any $f, f' \in F$, $f \succcurlyeq_w f'$ if and only if $\|\bm{x}_w - \bm{x}_f\|^2 \le \|\bm{x}_w - \bm{x}_{f'}\|^2$, and $f \succcurlyeq_w \perp$ if $\|\bm{x}_w - \bm{x}_f\|^2 \le t$, where $t$ is a threshold parameter. 
We fixed $t = 8$ in all the experiments.
Tie-breaking was performed by index order.

\paragraph{Evaluation Metrics.}
Here, we provide a formal definition of evaluation metrics. 
Because we only considered matchings with the same number of workers and firms in all experiments, we denote this number by $n$.

The Hamming distance between two matching matrices $\bm{M}$ and $\bm{M}'$ is $\sum_{i,j}|M_{ij} - M'_{ij}|$.
When $\bm{M},\bm{M}'\in\{0,1\}^{(n+1)\times (n+1)}$ are two matching matrices for $n$ workers and $n$ firms, the maximum value of Hamming distance is $3n$ where $\bm{M}$ represents a matching where every agent matches with $\perp$, and $\bm{M}'$ represents a matching with no agents being unmatched.
Therefore, we normalized the Hamming distance by dividing it by $3n$.

A matching matrix represents a unique matching, so the blocking pairs of a matching matrix are defined by the matching they represent according to Section~\ref{sec:background}.
The maximum number of blocking pairs is $n^2$; hence, we normalize the number of blocking pairs by dividing it by $n^2$.

Stability violation was proposed by \citet{ravindranath2023deep} to measure the degree of deviation from stable matchings under reports $\succcurlyeq$.
To define the stability violation, they expressed a reported profile $\succcurlyeq$ by rational numbers.
Formally, they defined $[p^\succcurlyeq_{i1},\dots,p^\succcurlyeq_{in}]$ and $[q^\succcurlyeq_{j1},\dots,q^\succcurlyeq_{jn}]$ to represent the reports $\succcurlyeq_{w_i}$ and $\succcurlyeq_{f_j}$, respectively, as follows:
\begin{align}
    &p^\succcurlyeq_{ij'} \notag\\
    &:= \frac{1}{n}\left(\mathbb{I}[f_{j'}\succ_{w_i}\perp] + \sum^n_{j''=1}(\mathbb{I}[f_{j'}\succ_{w_i}f_{j''}]-\mathbb{I}[\perp\succ_{w_i}f_{j''}])\right),\notag\\
    &q^\succcurlyeq_{ji'} \notag\\
    &:= \frac{1}{n}\left(\mathbb{I}[w_{i'}\succ_{f_j}\perp] + \sum^n_{i''=1}(\mathbb{I}[w_{i'}\succ_{f_j}w_{i''}] - \mathbb{I}[\perp\succ_{f_j}w_{i''}])\right),\notag\\
    \label{eq:preferece_vector}
\end{align}
where $a\succ_b a'$ means $a\neq a'~\land~a\succcurlyeq_b a'$.
Then, for a worker $w_i$ and firm $f_j$, they defined the stability violation of a matching represented by $\bm{M}$ at reports $\succcurlyeq$ as 
\begin{align*}
    &stv_{w_if_j}(\bm{M},\succcurlyeq) \\
    &:= \left(\sum_{i'=1}^n M_{i'j}\cdot\max\{q^\succcurlyeq_{ij} - q^\succcurlyeq_{i'j},0\}\right)\\
    &\phantom{:=} \times \left(\sum_{j'=1}^n M_{ij'}\cdot\max\{p^\succcurlyeq_{ij} - p^\succcurlyeq_{ij'},0\}\right).
\end{align*}
As a result, the stability violation on profile $\succcurlyeq$ is defined as 
\begin{align}
    stv(\bm{M},\succcurlyeq) := \frac{1}{n}\sum^n_{i=1}\sum^n_{j=1}stv_{w_if_j}(\bm{M},\succcurlyeq).
    \label{eq:def_stv}
\end{align}
\citet{ravindranath2023deep} proved that this value is 0 if and only if the matching is stable.
We measured this value for each test instance.
We normalized the stability violation by dividing it by $n$, considering that the summation is executed $n^2$ times and is already divided by $n$ in the definition.

\begin{figure*}[t]
    \centering
    \begin{minipage}[b]{0.23\linewidth}
        \centering
        \includegraphics[width=\linewidth]{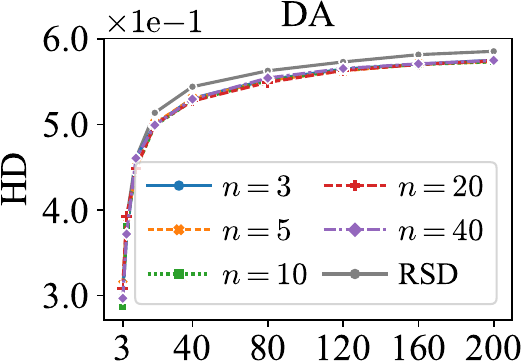}
    \end{minipage}
    \begin{minipage}[b]{0.23\linewidth}
        \centering
        \includegraphics[width=\linewidth]{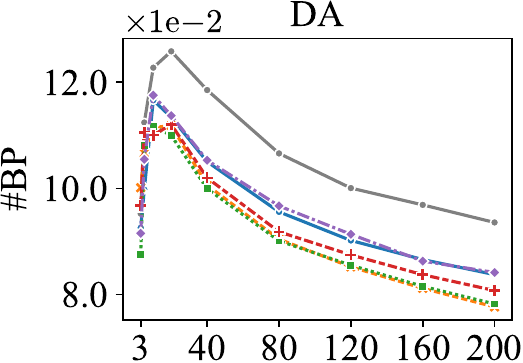}
    \end{minipage}
    \begin{minipage}[b]{0.23\linewidth}
        \centering
        \includegraphics[width=\linewidth]{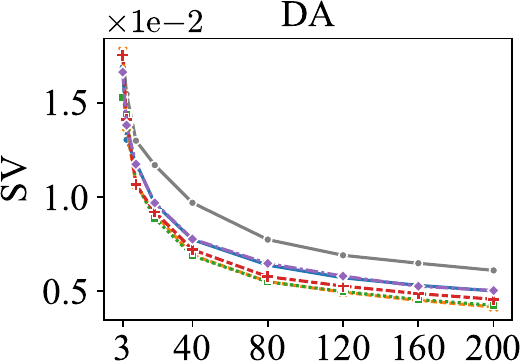}
    \end{minipage}
    \begin{minipage}[b]{0.23\linewidth}
        \centering
        \includegraphics[width=\linewidth]{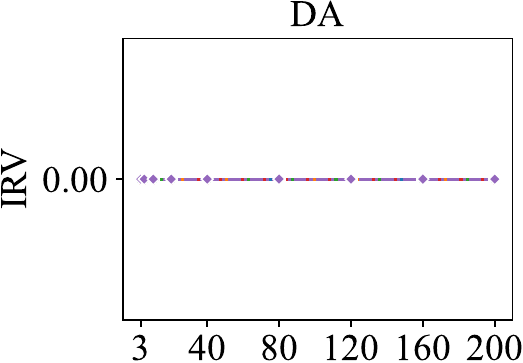}
    \end{minipage}
    \caption{
        Average performances across different scales for training and test data with stability regularization. 
    }
    \label{fig:test_cross_DA_lam}
\end{figure*}

IR violation on profile $\succcurlyeq$ is defined as~\cite{ravindranath2023deep}
\begin{align*}
    irv(\bm{M},\succcurlyeq) := &\frac{1}{2n}\sum^n_{i=1}\sum^n_{j=1}M_{i,j}\cdot\max\{-q^\succcurlyeq_{ij},0\}\\
    &+  \frac{1}{2n}\sum^n_{i=1}\sum^n_{j=1}M_{i,j}\cdot\max\{-p^\succcurlyeq_{ij},0\}.\label{eq:def_irv}
\end{align*}
This value reaches its minimum of $0$ when the matching is individually rational (IR), and its maximum of $1$ when each agent ranks their matched partner as the lowest.

The reward (RW) is defined by the value of the objective function in Equation~\eqref{eq:reward}.
Let $\succcurlyeq$ be a reported profile and $\bm{M}^\mathrm{opt}$ be the corresponding example matching, i.e., the solution of the problem~\eqref{eq:reward}.
Then, the reward for a predicted matching $\bm{M}$ is defined as the approximation ratio
\begin{align*}
    &\mathrm{RW}(\bm{M},\succcurlyeq) \\
    &:= \left(\sum^{n+1}_{i=1}\sum^{m+1}_{j=1}r^{h}(i,j)M_{i,j}\right) \biggl/ \left(\sum^{n+1}_{i=1}\sum^{m+1}_{j=1}r^{h}(i,j)M^\mathrm{opt}_{i,j}\right).
\end{align*}

\paragraph{Parameter Settings and Computational Resources.}
We set the attention embedding dimension $d_{\rm emb} = 10$ and SoftSort temperature parameter $\tau = 0.1$.
We trained models for 5 epochs on $n\in \{3,5,10,20\}$ with batch size 4, and for 10 epochs on $n=40$ with the same batch size using the Adam optimizer and gradient clipping with maximum $L_1$ norm at $10$.
The computational resources are listed in Table~\ref{tab:compute_resources}.
We conducted experiments on a macOS machine with an Apple M1 chip as the 16 GB machine and a CPU cluster with 256 cores of AMD EPYC 7763 processors as the 2 TB machine.

\subsection{Test Performance on Other Scales}
Supplementary results are presented in Table~\ref{tab:results_full}.
In this table, $n$ is equal in both the training and test data when $n\le 40$.
For test data with $n\ge 80$, we train NeuralSD on a dataset with  $n=40$.
The full results for all combinations of $n$ in training and test data are omitted due to space limitations.

\subsection{Experiments with Stability Regularization}
\label{apdx:stability_reg}
We conducted another experiments with examples generated by DA where the stability violation is added to the vanilla objective function.
We defined our loss function as
\begin{align}
    \mathrm{Loss}(\bm{\theta}) := \sum^L_{\ell=1} \left(\mathcal{L}(\hat{\bm{M}}^\ell, \bm{M}^\ell) + \frac{\lambda}{L} stv(\hat{\bm{M}}^\ell, \succcurlyeq^\ell)\right),
\end{align}
where $\hat{\bm{M}}^\ell :=\mathrm{NeuralSD}(\mathbf{P}_{W_{I^\ell}}, \mathbf{P}_{F_{I^\ell}}, \bm{X}_{W_{I^\ell}}, \bm{X}_{F_{I^\ell}})$ is the $\ell$-th prediction, $\mathcal{L}$ is the loss function defined in Equation \eqref{eq:defD}, and $\lambda$ is a hyper-parameter.
We set $\lambda=0.1$.
With slight abuse of notation, we denote by $stv(\hat{\bm{M}}^\ell, \succcurlyeq^\ell)$ the stability violation during training obtained by substituting (non-matching matrix) $\hat{\bm{M}}^\ell$ into Equation \eqref{eq:def_stv}.
The experimental results are provided in Table~\ref{tab:results_lam}, Table~\ref{tab:recovery_rate_lam}, and Figure~\ref{fig:test_cross_DA_lam}.

\section{Proof of Proposition~\ref{prop:irv_upper_bound_sd}}
\label{apdx:prop_irv_bound_sd}
We prove Proposition~\ref{prop:irv_upper_bound_sd} as follows.
\begin{proof}
    Let $n$ and $m$ be arbitrary, and consider an instance with $n$ workers and $m$ firms, where we assume, without loss of generality, that $n \geq m$.
    Fix an arbitrary agent ranking $\bm{r}$ for this instance, and let $\bm{M}_{\bm{r}}$ be the matching obtained by SD on $\bm{r}$.
    
    Each of the two agents matched in $\bm{M}_{\bm{r}}$ is either active in a round or passively matched with their counterpart (see Appendix~\ref{apdx:tensorserialdictatorship}).
    Let $\#_w$ and $\#_f$ be the numbers of workers and firms who are passively matched, respectively.
    The maximum number of pairs is $m$, so it holds that $\#_f + \#_w \le m$.
    
    IRV is maximized when the truthful preference profile satisfies the condition that all passively matched agents rank the counterpart as the lowest.
    This is because the terms in the summation are $0$ for unmatched and active agents, while they can be positive for passively matched agents, and lowering the ranks of counterparts for passively matched agents increases the overall sum.
    
    In this case, the value of $p$ or $q$ in Equation~\eqref{eq:preferece_vector} of the passively matched agents towards the counterpart is $-1$.
    As a result, IRV of $\bm{M}_{\bm{r}}$ on such a profile $\succcurlyeq$ is 
    \begin{align*}
        irv(\bm{M}_{\bm{r}}, \succcurlyeq) 
        &= \frac{1}{2n} \#_f + \frac{1}{2m} \#_w \\
        &\le \frac{1}{2m} (\#_f + \#_w) \\
        &\le \frac{1}{2m} \cdot m = \frac{1}{2}.
    \end{align*}
    This inequality holds for arbitrary $\bm{r}$ and $\succcurlyeq$ that satisfy the above condition.
    IRV for $\succcurlyeq$ that does not satisfy the condition has the same upper bound because the summation in Equation~\eqref{eq:def_irv} may decrease due to the ranks of passively matched agents.
    Therefore, the upper bound of IRV across all instances $I$ and agent rankings $\bm{r}$ is $1/2$.
\end{proof}

\section{Proof of Proposition~\ref{prop:computational_complexity_ranking}}
\label{apdx:prop_computational_complexity_ranking}
\begin{proof}
    The computational cost of this block totals $O((n+m)^2)$, comprising $O((n+m)^2 d_{\rm emb})$ for self-attention \cite{vaswani2017attention}, $O((n+m) d_{\rm emb})$ for one-dimensional linear transformations, $O((n+m)\log{(n+m)})$ for sorting in tie-breaking, and $O((n+m)^2)$ for $\mathrm{SoftSort}_\tau$.
    The total complexity is the sum of these four complexities.
\end{proof}

\end{document}